\RequirePackage{fix-cm} % For Springer AURO
\RequirePackage{amsmath}

\documentclass[10pt,journal,twocolumn]{IEEEtran} % IEEE Comp Sc.
\usepackage{url}
\usepackage{amsmath}
\usepackage{mathtools}

\usepackage{graphics}
\usepackage{graphicx}
\usepackage{epsfig}
\usepackage{epstopdf}

\usepackage[ruled,vlined,boxed,linesnumbered]{algorithm2e}
\SetArgSty{textnormal}
\SetAlFnt{\small}
\usepackage{algorithmic}
\newcommand*{\Break}{\textbf{break}}
\newcommand*{\End}{\textbf{end}}

\usepackage{amssymb}
\usepackage{float}
\usepackage{threeparttable}
\usepackage{cite}
\usepackage{subfigure}

\usepackage{units}
\usepackage{amsthm}

\newtheorem{proposition}{Proposition}
\usepackage{upgreek}
\usepackage{multirow}
\usepackage{hyperref}
\hypersetup{
colorlinks=false,
linkcolor=blue,
filecolor=magenta,
urlcolor=cyan,
}

\newcommand{\argmin}[1]{\underset{#1}{\operatorname{arg}\,\operatorname{min}}\;}

\newcommand{\appropto}{\mathrel{\vcenter{
  \offinterlineskip\halign{\hfil$##$\cr
    \propto\cr\noalign{\kern2pt}\sim\cr\noalign{\kern-2pt}}}}}

\begin{document}

%\title{Multi-robot Coverage via Boundary Shrink Control}
\title{Asymptotic Boundary Shrink Control with Multi-robot Systems}
\author{Shaocheng Luo, Jonghoek Kim, \and Byung-Cheol Min$^*$

\thanks{Manuscript accepted June 15, 2020.

This work was supported in part by NSF CAREER Award IIS-1846221.

S. Luo, and B.-C. Min are with SMART Lab, Department of Computer and Information Technology, Purdue University, West Lafayette, IN 47907, USA.

J. Kim is with Electrical and Electronic Convergence Department, Hongik University, Sejong-City, South Korea.

$^*$ Corresponding author email: {\tt minb@purdue.edu}.}}
%\markboth{IEEE Transactions on Systems, Man, and Cybernetics: Systems}%
%%%%%%% OPTION 1 %%%%%%%%%
\markboth{Preprint version. Published in IEEE Transactions on Systems, Man, and Cybernetics: Systems - \url{https://doi.org/10.1109/TSMC.2020.3003824}}%
%%%%%%% OPTION 2 %%%%%%%%%
%\markboth{Preprint version accepted to IEEE Transactions on Systems, Man, and Cybernetics: Systems - \url{https://doi.org/10.1109/TSMC.2020.3003824}} %%%This link will become available once our paper is available online in IEEE I think.
{\MakeLowercase{\textit{Luo and Min}}:}

\setlength{\textfloatsep}{8pt}% LUO
%\IEEEtitleabstractindextext{
\maketitle

%This is to remove the page num of the first page, can be removed 
%\thispagestyle{empty} 

\begin{abstract}
Harmful marine spills, such as algae blooms and oil spills, damage ecosystems and threaten public health tremendously. Hence, an effective spill coverage and removal strategy will play a significant role in environmental protection. In recent years, low-cost water surface robots have emerged as a solution, with their efficacy verified at small scale. However, practical limitations such as connectivity, scalability, and sensing and operation ranges significantly impair their large-scale use. To circumvent these limitations, we propose a novel asymptotic boundary shrink control strategy that enables collective coverage of a spill by autonomous robots featuring customized operation ranges. For each robot, a novel controller is implemented that relies only on local vision sensors with limited vision range. Moreover, the distributedness of this strategy allows any number of robots to be employed without inter-robot collisions. Finally, features of this approach including the convergence of robot motion during boundary shrink control, spill clearance rate, and the capability to work under limited ranges of vision and wireless connectivity are validated through extensive experiments with simulation.

\end{abstract}

%\maketitle % IEEE

% \fi

%\IEEEdisplaynontitleabstractindextext %IEEE
%\IEEEpeerreviewmaketitle %IEEE

%\IEEEraisesectionheading{\section{Introduction} %IEEE

%\vspace{-20pt}
\section{Introduction}
\label{sec:intro}
\IEEEPARstart{C}{overage} operation has a wide application scope in environment searching and exploration. Currently, methods of coverage control have seen tremendous potential in solving environmental issues caused by side effects of urbanization and industrialization. 
One global issue of increasing importance is water pollution, including harmful algae blooms that result from eutrophication, and oil leakage in the oceans during transportation, due to its connections with potable water supplies and even marine ecosystems. 
It is a natural idea to deploy a team of water surface robots for cleaning such pollution, in order to protect human operators and to provide for immediate operation response before any ecological or economic damage results.

\begin{figure}[t]
\centering
\includegraphics[width=0.9\columnwidth]{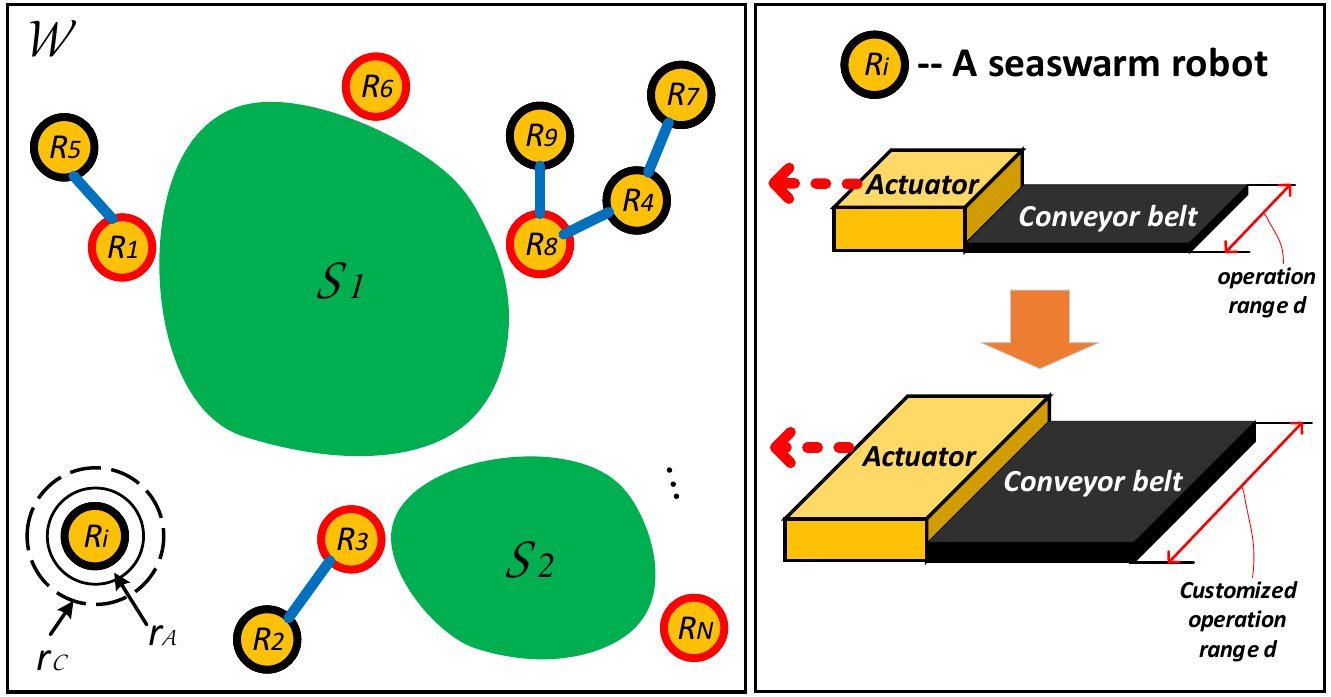}
\caption{A demonstrative figure showing the workspace $\mathcal{W}$ of spill coverage, where two spills, $\mathcal{S}_1$ and $\mathcal{S}_2$, are surrounded by $N$ number of distributed water surface robots such as MIT Seaswarm robots \cite{mit2010}. We consider to advance the Seaswarm robot by introducing a customized operation range $d$ to adapt different operation scenarios, whose control and multi-robot coordination scheme is implemented in our proposed strategy. For each robot $R_i$, $r_\mathcal{A}$ and $r_\mathcal{C}$ represent its vision sensor range and the wireless communication range, respectively. The robots with red circles are in the vicinity of the spills, and the blue line represents a wireless connection. Note that the ranges and sizes of robot are not in real-scale.}
\label{fig:seaswarm}
\end{figure}

Researchers from MIT developed one approach for the removal of oil spills in the ocean: Seaswarm, a fleet of low-cost, oil-absorbing robots \cite{mit2010}. These robots are powered by solar cells and designed to move on the water surface, absorbing oil spills through the long edge of an attached conveyor belt, as shown in Fig.~\ref{fig:seaswarm}. The Seaswarm robot has proven to be an effective prototype for spill cleaning operations; however, hardly any relevant research can be found that discusses control and collaboration strategies for this class of robot.

Outdoor coverage operations usually rely heavily on a global coordinate or external localization system, such as the Global Positioning System (GPS) \cite{clark2005cooperative}. However, robot control based on  GPS sensors suffers from a low update rate ($\leq$ 1 Hz), large and high-variance localization errors, and signal loss due to harsh weather. 
Moreover, in military scenarios, there may be cases where the GPS signal is jammed or not available due to an enemy's attack. Therefore, we develop a networked multi-robot system which does not rely on such external signals.

In this work, we propose an asymptotic boundary shrink control method using a team of robots that can collaboratively remove spills in diverse environments. We do not assume a global view of the workspace; instead, the robot team will first perform a random walk until each spill has been detected by at least one robot \cite{dhariwal2004bacterium}. For each spill, the nearest robot becomes a rendezvous point for other robots by way of wireless connectivity mediated through on-board communication modules \cite{luo2019multi}. When any robot finishes migrating toward the rendezvous point to which it is allocated, this robot will begin boundary shrink control. When a boundary shrinks to zero and the robots meet each other, the associated spill has been covered and removed. During the boundary shrink process, we use local monocular vision sensors with a single line detection method to reduce reliance on GPS sensors and obtain more accurate control input. 

To the best of our knowledge, the boundary shrink control method proposed in this work is the first to enable collective spill removal operation with multiple robots similar to Seaswarm \cite{mit2010}, while also considering requirements such as adjustable robot operation ranges and constraints including limited vision and communication ranges, unavailability of global positioning systems, and collision avoidance. In addition to marine operations, the proposed coverage strategy can be applied in situations such as lawn mowing, forest fire-fighting, humanitarian de-mining, and more. As it does not rely on an external localization system, our strategy can even be applied for indoor coverage control, searching, and explorations. 

Additionally, the distributedness of our method can tackle coverage situations where multiple discrete spills exist in the workspace. The efficacy of concurrently performing boundary shrink control on multiple spills is demonstrated in our simulations in Sec.~\ref{sec:experiments}. The proposed strategy achieves clearance rate of $>99\%$ on multiple spills and with a variety of robot team settings.

The main contributions of this paper are threefold: 
\begin{itemize}
\item This paper proposes a boundary shrink control method without path planning or localization;
\item The proposed method can realize complete coverage of multiple spills in a distributed way, under practical constraints including customized robot operation ranges and limited communication and vision sensing ranges; 
\item The robot team allows strong scalability of the method.
\end{itemize}

The proposed coverage control strategy can also be applied in situations such as water quality monitoring \cite{luna2017robotic}, pollutant source searching \cite{liu2016automated}, lawn mowing, forest fire-fighting, humanitarian de-mining, and others.

The paper is organized as follows: Sec.~\ref{sec:work} reviews the relevant work and compares our method with the state-of-the-art research. Sec.~\ref{sec:problem} defines the research problem under the necessary assumptions and presents the proposed solution. Sec.~\ref{sec:modeling} describes the control model for the robot, which is followed by a boundary shrink model, coverage manifold, and bounded moving speed during coverage based on robot processing capability. A multi-point multi-robot rendezvous strategy for boundary shrink control with limited wireless communication ranges is detailed in Sec.~\ref{sec:rendezvous}, while the boundary shrink control method is elaborated in Sec.\ref{sec:shrink_control}. Experiments with simulation and analyses of significant properties including boundary evolution, spill clearance rates, robot convergence, and operation time are elaborated in Sec.~\ref{sec:experiments} using the Robotarium testbed in MATLAB. Spills with non-convex coverage present challenges to the proposed solution, and these challenges are discussed in Sec.~\ref{sec:nonconvex_spill}. This work is concluded in Sec.~\ref{sec:conclusions}.

\section{Related Work}
\label{sec:work}
Multi-robot systems have recently exhibited strong capability in environmental operations including cleaning oil spills \cite{jin2014navigation} and fighting forest fires \cite{ghamry2017multiple}, which implicitly involves collective coverage of the area of interest. 

The majority of research into multi-robot coverage operation focuses on path planning and robot distribution strategies, which are enabled by assuming a known environment or defined workspace. Coverage control using trapezoidal decompositions and Morse decompositions toward a defined area is discussed in \cite{choset2005principles} and \cite{acar2002morse}, respectively, where a robot was designated to take care of a specific small area and travel in straight lines. Coverage of a water area was presented in \cite{li2016underwater}, which required that each robot knows the planned path and its own accurate location. The work in \cite{an2017rainbow} proposed to achieve target region coverage by way of motion planning based on triangulation. Multi-robot coverage with Boustrophedon decomposition is discussed in \cite{kong_distributed_2006}, which requires a pre-partitioned workspace for each robot as well as incessant localization.
Strategies using pre-planned paths or motions can result in higher efficiency and less computational consumption for each robot, but in practice, the global coordinate frame and localization techniques for robot trajectory tracking are not always available, making tracking along the planned path unattainable. Moreover, the scale and shape of a real-world spill are usually unknown to the robot team, and coverage methods that depend on pre-planning are less adaptive to unknown environments or abrupt changes during the course of the operation. 

In light of the aforementioned problems, our proposed method does not rely on any planning or global localization systems, such as \cite{zhou2017real,cui2015mutual}, and therefore can be applied in unknown or GPS-denied environments. We utilize a local monocular vision sensor with which each robot detects and tracks the spill boundary. However, the neighboring robot collision avoidance can be achieved using sonar \cite{luo2019multi}, radar, or lidar sensors with 360 degrees scanning \cite{gohring2011radar}. Therefore, the vision field or blind zone of vision sensors will not hinder the robot collision avoidance. A robot does not need to share its boundary detection information with others to construct a global map, as was done in \cite{zhang2010cooperative}; instead, it can perform the coverage operation with a simple tracking model. Additionally, unlike other work involving the detection of a two-line boundary (left and right),  similar to a driveway \cite{monocular,Salamanca}, our method enables single line detection and so avoids massive computation.

Furthermore, many aforementioned coverage control methods feature centralized control schemes, such as those of \cite{choset2005principles,acar2002morse}. The Boustrophedon decomposition in \cite{kong_distributed_2006} requires a pre-partitioned workspace for each robot, which limits its application scope. With centralized control, coverage performance depends largely on the master robot, whose decisions are based on feedback information from all slave robots. 
During operation, robots cannot be added to or removed from the task, meaning the system does not scale to different numbers of robots nor adapt to faulty robots. Our method advances the literature by nature of its complete distributedness, where each robot in a team can choose its own working state independent of its neighbours. Thus, the control scheme need not be updated according to the number of deployed robots.

Distributed coverage control methods based on Voronoi diagrams and Lloyd’s algorithm have drawn attention in recent years. In papers such as \cite{schwager2011unifying,alonso2011multi}, a team of robots is used to cover a specific area in a static manner once each robot is deployed to its goal position. A strategy based on a pivot robot is developed in \cite{luo2018pivot}. The cooperative surveillance strategy from \cite{casbeer2006cooperative} can be used in spill removal, but it requires massive inter-robot communications and robust connections. The work in \cite{qu2014finite} proposed a moving coverage pattern where each agent $i$ in the team has a specific sensor model and a defined sensory domain. Notably, all of these methods are primarily designed for information coverage, having multiple agents equipped with sensors, rather than for physical coverage such as is needed for spill removal. The works mentioned above, especially  \cite{qu2014finite,jin2014navigation}, assumed each robot to be a point mass and did not elaborate the kinematics and control of the robots. Furthermore, they did not take into account practical limitations such as designated operation ranges, limited sensing ranges, and impediments to wireless connectivity. Our proposed boundary shrink control solution shows efficacy under these practical limitations, and is verified by extensive simulations referring to a Seaswarm robot prototype \cite{mit2010}.

Our method also addresses problems such as the existence of multiple discrete spills in the workspace and collision avoidance for collaborating robots, which further strengthens the significance of our solution.

\section{Problem Statement and Mathematical Foundation}
\label{sec:problem_and_math_foundation}

This section structurally formulates the research problems in a mathematical way with reasonable assumptions.

\subsection{Problem Statement}
\label{sec:problem}
Given a workspace denoted as $\mathcal{W} \in \mathbb{R}^2$, as shown in Fig.~\ref{fig:seaswarm}, this study demonstrates how a certain number of randomly distributed robots can be controlled to implement coverage operation in the case of constraints such as limited vision sensing ranges and limited wireless signal ranges, while adapting to customized operation ranges. Meanwhile, this study reveals the optimal partitioned to the robots in a graph with several disconnected components and explains how they can collaborate in shrinking the spill boundary and complete coverage with only local image sensing. 

As to the workspace, it is obstacle-free except for robots, and it is possible that many spills exist in the area of interest.
Avoiding collision between robots coming from different spills can be handled by letting each robot use simple local collision avoidance controller, e.g., \cite{Cone} or applying techniques such as stigmergic communication \cite{ranjbar2012multi}. Thus, it is not within the research scope of this paper. 

We assume a sufficient number of robots deployed in the coverage operation. However, if only a few robots are available, for instance less than the number of spills, they can be redeployed to cover the remaining spills by applying the boundary shrink control method. For each robot in a team, it begins a random walk from the same dock as the others and stops once detecting the boundary. The dock cannot be inside a spill, otherwise no robot can move outward to detect other spills. Furthermore, at the time when the random walk is ceased, no robot can be inside a spill, because it should have stopped before traversing the boundary. We also assume that those robots that cannot see the spill are wirelessly connected to at least one other robot that can see the boundary.

As a solution, this paper first introduces a unicycle model for robot control, shown in Fig.~\ref{fig:dynamics}. Based on this model, a hybrid control strategy is proposed for the robot team.
The motion of each robot is represented as a finite state machine consisting of three individual states, as shown in Fig.~\ref{fig:state_transition}, with the collision avoidance hierarchy being prioritized as {\it Tracking \textgreater \, Searching \textgreater \, Rendezvous}. 
A more detailed illustration of these states for one spill scenario is found in Fig.~\ref{fig:cycle}. The task is completed once there is no spill within the view of the robot.

\subsection{Mathematical Foundation for Boundary Shrink Control Method and Robot}
\label{sec:modeling}
In order to show the performance of the robot in coverage, we first need to elaborate a robot control model and then robot control laws with a bounded speed. In this sub-section, the boundary shrink control model for spill coverage is presented.

\begin{figure}[t]
\centering
\includegraphics[width=0.75\columnwidth]{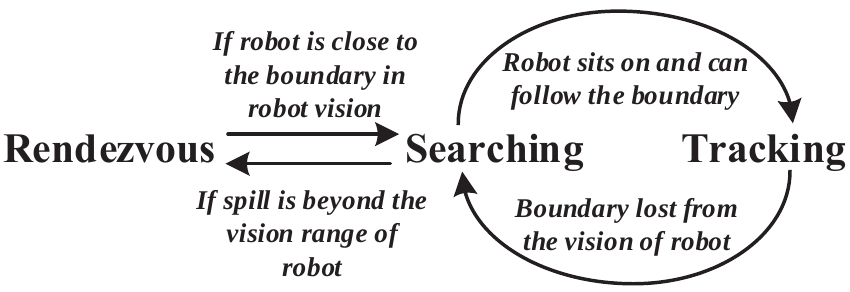}
\caption{The state diagram of the proposed multi-robot boundary shrink control method, and the state transition of robots in workspace $\mathcal{W}$. The collision avoidance hierarchy is prioritized as {\it Tracking \textgreater \, Searching \textgreater \, Rendezvous}.}
\label{fig:state_transition}
\end{figure}

\subsubsection{Vehicle Model}
The driving system that actuates the robot is composed of two unidirectionally placed parallel actuators (e.g., thrusters), which give the robot mobility with three degrees of freedom (DOF): x- and y-translation, and z-rotation. The robot control model plays an important role in determining the motion control laws. Assuming the center of buoyancy coincides with the center of gravity, and they are balanced with each other, it is concluded that all the robots move in a two-dimensional space $\mathbb{R}^2$.
$R_i$ denotes the $i$-th robot, here $i\in \{1,..., N\}$.
Let the generalized coordinate of $R_i$ be $\textbf{q}_i=(x_i,y_i,\theta_i)$, we introduce a unicycle model \cite{zhang2013spill}, which has been validated previously through field tests using surface vehicles  \cite{8962837,jo2019development}, as below for the robot actuators shown in Fig.~\ref{fig:dynamics}.  
\begin{equation}
\dot{\textbf{q}}_i = \begin{bmatrix} \dot{x}_i \\ \dot{y}_i \\ \dot{\theta}_i \end{bmatrix} = \begin{bmatrix} \cos{\theta_i} & 0 \\ \sin{\theta_i}  & 0 \\ 0 & 1 \end{bmatrix} \begin{bmatrix} v_i \\ \omega_i \end{bmatrix} = \textbf{J}(\theta_i)\textbf{Q}_i
\label{equ:kine_1}
\end{equation}

\begin{equation}
\textbf{Q}_i = \begin{bmatrix} v_i \\ \omega_i \end{bmatrix} = 
\begin{bmatrix} \frac{1}{2} (v_{ri} + v_{li}) \\ 
\frac{1}{L} (v_{ri} - v_{li}) \end{bmatrix} =
\begin{bmatrix} \frac{1}{2} & \frac{1}{2}  \\ 
\frac{1}{L} & -\frac{1}{L} \end{bmatrix}
\begin{bmatrix} v_{ri} \\ v_{li} \end{bmatrix} 
\label{equ:kine_2}
\end{equation}
where $\textbf{Q}_i$ is the velocity vector and control objective of robot $R_i$ consisting of linear velocity $v_i$ and angular velocity $\omega_i$ with respect to the centroid. $L$ in (\ref{equ:kine_2}) denotes the tread width of the robot. 
The angular velocity $\omega_i$ is bounded by $\omega_{max}$,
%The angular velocity $\omega_i$ is fixed for some constant $\omega_{cst} >0$.
%$\omega_i \in \{-\omega_{cst}, 0, \omega_{cst}\}$ for some constant $\omega_{cst} >0$. 
which will be determined later in Sec.~\ref{sec:tracking_peri_cyc}.
%determined according to hardware constraints. 
Furthermore, the angular controller for the robot is discrete in time and its resolution is $\phi_{ci}=\omega_iT_c$, where $T_c$ is the system cycle. 
%Hence, an robot can move forward with $v_i$, and turns clockwise/counterclockwise $(\omega_i)$ about the robot centroid. 
Since the robot is driven by the two parallel actuators, the control objective space of $\textbf{Q}_i$ has to be transformed into the space of $v_{li}$ and $v_{ri}$. Additionally, we formulate control input as $[u_i, w_i]$, to differentiate from the velocity $[v_i,\omega_i]$.
$[u_i, w_i]$ is the control input set as a ``high level control", and $[v_i, \omega_i]$ represents the actual motion of the robot due to the control input.

\subsubsection{Boundary Shrink Model}
\label{sec:shrink_model}

\begin{proposition}
\label{prop:shrink_model}
When a team of robots performs coverage of a spill collectively via  tracking the boundary, each robot with an effective range $d$, an asymptotic boundary shrink control can be realized. 
\end{proposition}

\begin{proof}
Trivial. As shown in Fig.~\ref{fig:model}, when robot $R_i$ travels along the boundary counterclockwise (CCW) and covers the spill on its left hand side, the unit area that is covered in time $\Delta t$ is $\Delta\mathcal{C}$. When the robot first penetrates into a closed spill and moves from point $p(t_0)$ to $p(t_1)$, as illustrated in Fig.~\ref{fig:model}, the boundary length changes from $s_i(t)=\Delta\varphi\cdot r_i(t)$ to $s_i'(t)+2d=\Delta\varphi\cdot (r_i(t)-d)+2d$, assuming the radius of curvature $r_i$ of the new position $p(t_i)$ is available. 
Even if the length of boundary increases aperiodically at some point, it shrinks to zero when the entire spill is removed, which leads to an asymptotic boundary shrink effect. 
Assuming $N$ robots are involved in the boundary shrink control, the boundary will shrink asymptotically with a greater gradient while the robot team performs coverage. This ends the proof.
\end{proof}

\subsubsection{Bounded Speed and Acceleration}
\label{sec:bounded_speed}
When a robot is tracking the boundary of a spill denoted as $\partial\mathcal{S}$, and visits points $\text{A}(x_1,y_1)$ and $\text{B}(x_2,y_2)$ in sequence, we propose an approximated coverage manifold $\Delta\mathcal{C}'$, which is shown in Fig.~\ref{fig:model} and formulated as below in (\ref{equ:bounded_speed_1}) and (\ref{equ:bounded_speed_2}), to decide the maximal speed of robots. A similar model is used in \cite{luo2018pivot}.
\begin{multline}
\label{equ:bounded_speed_1}
\Delta\mathcal{C}' = \{\text{M}(x,y):(0 < \textbf{AM} \cdot \textbf{AB} < \textbf{AB} \cdot \textbf{AB})\, \cap \\ (0 < \textbf{AM} \cdot \textbf{AD} < \textbf{AD} \cdot \textbf{AD})\}, 
%D(x_4, y_4) = \bigg(-\sqrt{|\textbf{AD}|^2 - y_4^2} \, , \, \frac{|\textbf{AB}|^2 + |\textbf{AD}|^2 - |\textbf{BD}|^2}{2|\textbf{AD}|} \bigg),
\end{multline}
\begin{equation}
\label{equ:bounded_speed_2}
\text{D}(x_4, y_4) = \big(-\sqrt{\|\textbf{AD}\|^2 - y_4^2} \, , \, \frac{\|\textbf{AB}\|^2 + \|\textbf{AD}\|^2 - \|\textbf{BD}\|^2}{2\|\textbf{AD}\|} \big)
\end{equation}
where $\text{M}$ is a point in covered area $\Delta\mathcal{C}'$, and $\|\textbf{AD}\|=d$ 
%$|\textbf{BD}|=\sqrt{(v\Delta t)^2 + d^2}$, 
is the effective range. Additionally, assuming a limit $\mathcal{V}$ for processing capacity in spill removal, which is similar to centrifugal volume and filtering system capacity in collecting contaminants, by having $\Delta\mathcal{C}' = p_i\cdot d = vd\Delta t = \mathcal{V}\Delta t$, we determine the bounded speed for robots as below:
\begin{equation}
v_{max} = f(\mathcal{V},d) = \frac{\mathcal{V}}{d}.
\label{equ:v_max}
\end{equation}
Here, we further set a maximum acceleration $a_{max}$ which is bounded by actuator output power. When the robot moves, its linear acceleration will follow:
\begin{equation}
|a_i| \leq a_{max}.
\label{equ:a_max}
\end{equation}

The coverage model (\ref{equ:bounded_speed_1}) and the effective range $d$ can be validated with existing robots such as \cite{mit2010}, which is also shown in Fig.~\ref{fig:seaswarm}. A left hand side coverage style leads to a uniform counterclockwise robot team moving direction which helps avoid inter-robot collision in a distributed way. The elaborated details can be found in Sec.~\ref{sec:shrink_control}. 

\vspace{-0.3cm}
\begin{figure}[t]
\centering
\begin{tabular}{cc}
\subfigure[]{\label{fig:dynamics}
\includegraphics[width=0.2\textwidth]{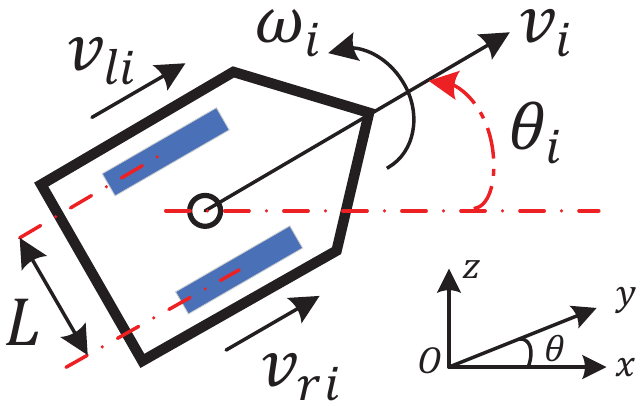}} &
\subfigure[]{\label{fig:model}
%\hspace{0.3cm}
\includegraphics[width=0.22\textwidth]{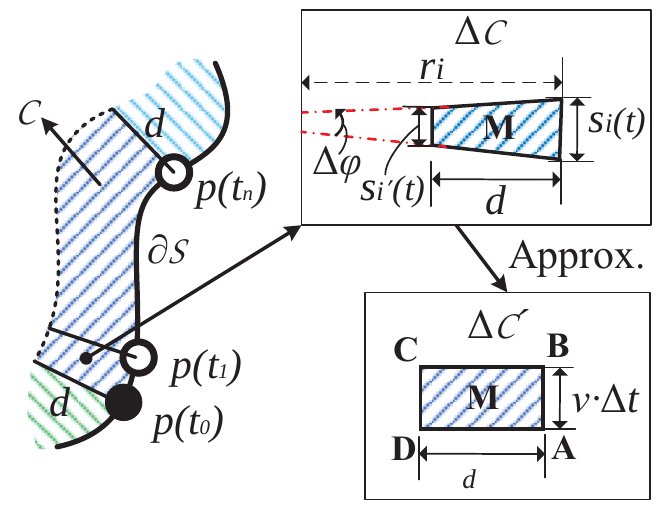}} 
\end{tabular}
\caption{(a) The mobility chassis of the robot driven by two parallel actuators. The robot has three degrees of freedom, i.e., $\dot{x}_i,\, \dot{y}_i, \text{and} \,\omega_i$.  (b) Coverage manifold for a robot maneuvering by tracking $\partial\mathcal{S}$. A unit area covered is shown in a subfigure named $\Delta\mathcal{C}$, and its approximation is shown in another subfigure named $\Delta\mathcal{C}'$ with unit time $\Delta t$ and speed $v$.}
\label{fig:two_models}
\end{figure}

\section{Multi-robot Rendezvous}
\label{sec:rendezvous}

This paper handles the scenario where multiple separate spills exist in the workspace. 
With a limited vision sensing range, each spill can be recognized by robots using our previous multi-point rendezvous algorithm in \cite{parasuraman2018multipoint}, provided that every spill has at least one robot at its vicinity at first. As long as we deploy sufficiently many robots initially, this assumption is feasible. 

At the beginning of the operation, every robot that can see spill boundary is set as a rendezvous point $D_m \in D$ $(m\leq \mathcal{M})$. Here, $\mathcal{M}$ is the number of rendezvous points, and $D$ is the set of all rendezvous points. Then, multi-point rendezvous algorithm in \cite{parasuraman2018multipoint} can be applied and every robot can rendezvous to its associated rendezvous point. Such rendezvous process is distributed and assures collision avoidance among robots. 
\footnote{We acknowledge that this paper uses the multi-robot rendezvous algorithm in our previous paper\cite{parasuraman2018multipoint}.
Thus, introducing a new multi-robot rendezvous algorithm is not within the scope of this paper.}

When the rendezvous process is completed under the algorithm in \cite{parasuraman2018multipoint}, every robot will sit near its corresponding rendezvous point and therefore be able to detect the boundary. Once a robot meets its associated rendezvous point, the robot will switch to boundary searching state and move toward the nearest point on the boundary, as shown in Fig.~\ref{fig:state_transition} and \ref{fig:cycle}.

Let $A(k)=(V(k), E(k), W(k))$ denote the undirected graph for wireless connectivity at time step $k$, where $V(k)$ represents the vertex set and $E(k)$ the edge set. The weight $w_{ij}(k)$ of each edge in the weight matrix $W(k):[w_{ij}(k)]$ is approximately proportional to the Euclidean distance between two robots (i.e., $w_{ij}(k) \appropto \|\textbf{q}_i(k)-\textbf{q}_j(k)\|$), which can be determined by wireless radio signal strength (RSS) measurement \cite{luo2019multi}. Note that the graph $A(k)$ is established based on local interaction of each robot by running {\it Algorithm 1} in \cite{parasuraman2018multipoint}.
Since Algorithm 1 in \cite{parasuraman2018multipoint} is distributed, every robot can build  $A(k)$ in a distributed manner.

According to \cite{parasuraman2018multipoint}, each robot in the team $R_i \in \mathcal{R}$ will rendezvous to the leader robot $P(R_i) \in D$ with the least amount of time and energy, by following the optimization function below:
\begin{equation}
    P(R_i) = \argmin{D_m \in D} d_{A(k)}(R_i, D_m),
\label{equ:opt_rend}
\end{equation}
where $d_{A(k)}(R_i, R_j)$ in the function denoting the shortest path between two nodes $R_i$ and $R_j$ in the graph $A(k)$ using Dijkstra's algorithm. Ultimately, $\mathcal{M}$ subgraphs ($A_m\subset A(k)$ where $m\in\{1,2,...,\mathcal{M}\}$) are constructed for each rendezvous point, which is the node in the vicinity of the spills.

For each subgraph $A_m\subset A(k)$, a shortest-path-tree $T_m$ is constructed with $D_m$ as the rendezvous point. A hierarchical rendezvous algorithm, i.e., {\it Algorithm 2} in \cite{parasuraman2018multipoint}, is utilized to enable such rendezvous in each $A_m$ such that all the nodes in this graph can be gathered around the associated spill. Such hierarchical rendezvous is validated through real experiments in \cite{luo2019multi}.

In the hierarchical rendezvous algorithm ({\it Algorithm 2}) of \cite{parasuraman2018multipoint}, the child robot updates its parent from the current to the one in the upper level, when the child robot is very close to the current parent. This fact implies that the leaf robot (i.e., a node without any child) can switch to boundary searching and tracking state when it meets the rendezvous point $D_m$, which therefore closes the loop connecting multi-robot multi-point rendezvous and the proposed boundary shrink control.

In case that the spill is moving, the rendezvous points also move accordingly to avoid losing the vision contact to the boundary. Then, all the associate child robots will have to move following the rendezvous point to preserve wireless connectivity. This feature, namely herding in the literature \cite{parasuraman2018multipoint}, guarantees the success of the boundary shrink control under dynamic spill situation while avoiding creating an additional state to those shown in Fig.~\ref{fig:state_transition}.

Compared with the work in \cite{parasuraman2018multipoint}, this work advances in extending the system distributedness. Instead of assuming a connected robot team in the workspace as described in \cite{parasuraman2018multipoint}, we validate cases where multi-point rendezvous succeeds in a graph with several disconnected components, as shown in Sec.~\ref{sec:case_2}. A disconnected graph is more natural and practical for distributed robotic systems with limited wireless communication range.

\section{Boundary Shrink Control Scheme}
\label{sec:shrink_control}
This section discusses how robots perform boundary shrink control by searching and tracking the boundary of a spill, and what controllers are designed in order to both avoid collision and achieve a complete coverage over the spill. Boundary searching and tracking happen simultaneously, thus they are discussed together here.

\subsection{Boundary Searching Controller and  Algorithms}
\label{sec:tracking_peri_cyc}
When each robot is in place, it will then move along the spill boundary with the interior of the spill on the left hand side. In other words, all the robots move counterclockwise around the boundary, as demonstrated in Fig.~\ref{fig:cycle}, in the boundary tracking state. For a robot in rendezvous to begin tracking state and perform coverage operation, it has to switch to searching state first and align itself to the direction of the boundary. The boundary tracking control strategy is elaborated in Algorithm~\ref{alg:B_searching}, while its initialization is detailed in Algorithm~\ref{alg:B_searching_init}.

\begin{figure}[t]
\centering
\includegraphics[width=0.99\columnwidth]{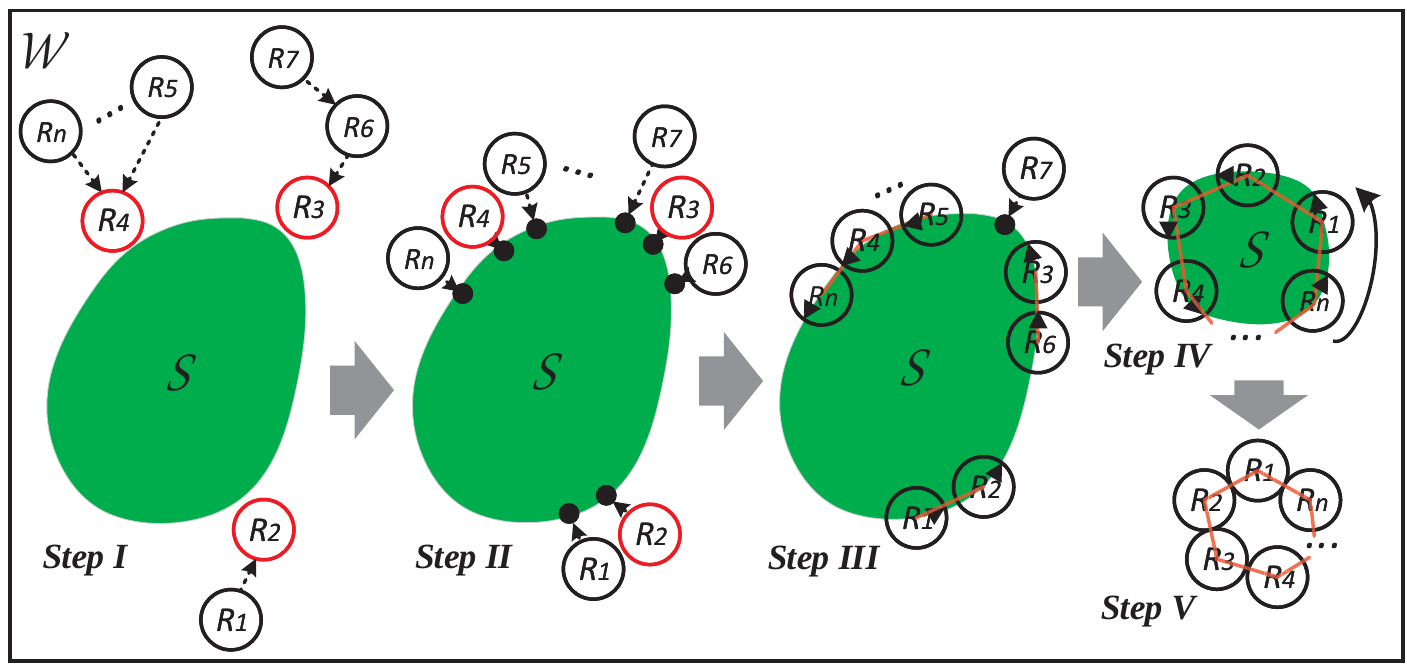}
\caption{A representative figure showing our proposed collective boundary shrink control strategy. The multi-robot hierarchical rendezvous is demonstrated in {\it Step I} where robots serving as rendezvous points $D_m$ are highlighted with a red circle and the rendezvous hierarchy is indicated with dashed arrows. The robots deployed shall start boundary searching and navigate to the goals when they are in the vicinity of the spill and have no child robot associated, as shown in {\it Step II}.  {\it Step III} demonstrates boundary tracking state on the completion of boundary searching, and robots maneuver along the boundary to shrink the spill. When robots are moving sequentially and getting closer, the proposed strategy prevents collision, as depicted in {\it Step IV}. The task is finished in {\it Step V}, where boundary shrinks to zero and no spill remains. The black dots denote the nearest deployment positions that robots are aiming for in the searching state. The triangles indicate robot moving direction, while the red lines mean the leading and trailing robots reside within APF effective range.}
\label{fig:cycle}
\end{figure}

The goal position on the spill boundary can be determined via local vision sensors such as cameras, where the traveled distance of robots is minimized. The shortest distance $\|d_i\|_{min}$ between $q_{s,i}$, the start position of robot $R_i, i \in {1, 2,...,N}$, and $q_{g,i}$, the goal position which lays on the edge of spill $\mathcal{S}$, can be estimated after the spill boundary is identified by image processing techniques \cite{konolige2008outdoor}. Given the coordinates of goal positions for robots in the workspace, our study applies the artificial potential field (APF) method to enable navigation until robots reaches the goal positions. The navigation can be achieved with a robot onboard vision sensor that measures the current distance to the goal. An attractive potential field is introduced to achieve navigation and a repulsive potential field is introduced for collision avoidance. The robots working in the tracking state are assumed to have a higher avoidance priority than those maneuvering for rendezvous or searching. The avoidance priority is reflected in Fig.~\ref{fig:state_transition}. If the workspace consists of multiple spills that need to be covered by the robots, an appropriate partition to the robot team has to be decided in completing multi-point rendezvous algorithm. The searching and tracking to the boundary remain the same for multiple spills cases as a single spill scenario.

Let $\textbf{q}_{g,j}\in\mathbb{R}^2$ denote the goal position of a robot $R_j$. Let $d(\textbf{A}, \textbf{B})$ denote the distance between two positions $\textbf{A}\in \mathbb{R}^2$ and $\textbf{B}\in\mathbb{R}^2$.
The attractive potential for a robot $R_j$ moving toward the goal is formulated as
\begin{equation}
\label{equ:attractive_APF_dep}
U_d(\textbf{q}_j) = \frac{1}{2}\xi_1 d(\textbf{q}_j, \textbf{q}_{g,j})^2
\end{equation}
where $\xi_1$ is a scaling parameter.
Meanwhile, a repulsive potential formulated as (\ref{equ:repulsive_APF_dep}) is exerted on the robots working in searching state by the robots in tracking state, in order to avoid collision.
\begin{equation}
\begin{gathered}
U_i(\textbf{q}_j) = \begin{cases}
\frac{\xi_2}{2}  \Big(\frac{1}{d(\textbf{q}_j, \textbf{q}_i)} - \frac{1}{d_0} \Big)^2, & \text{if}\, d(\textbf{q}_j, \textbf{q}_i) \leq d_0, \\
0, & \text{if}\, d(\textbf{q}_j, \textbf{q}_i) > d_0, 
\end{cases} \\
i \in \mathcal{N}_t
\label{equ:repulsive_APF_dep}
\end{gathered}
\end{equation}
where $\xi_2$ is a positive scaling factor.
$\mathcal{N}_t$ denote the robots working in tracking state. Here $d_0$ is the effective range, and we further assume that all the robots have equal $d_0=r_{\mathcal{A}}$. 

Due to this, for any robot $R_j$, the final constructed potential field for searching navigation is:
\begin{equation}
\label{equ:sum_APF_dep}
U^*(\textbf{q}_j) = U_d(\textbf{q}_j) + \sum_{i \in \mathcal{N}_t} U_i(\textbf{q}_j)
\end{equation}

From the potential field function (\ref{equ:attractive_APF_dep}), and assuming no other robot within scope $d_0$, we obtain the control input for the robots as below:

\begin{subequations}
\label{equ:deployment_vel}
\begin{align}
\textbf{u} &= k_u\tanh (\|\textbf{q}_g-\textbf{q}\|^2), \label{equ:deployment_vel_u}\\
\textbf{w} &= -k_w(\boldsymbol{\uptheta}-\boldsymbol{\upvarphi})+\dot{\boldsymbol{\upvarphi}}.  \label{equ:deployment_vel_w}
\end{align}
\end{subequations}

Here, $\boldsymbol{\uptheta}=[\theta_1,\theta_2,...,\theta_N]$, and $\boldsymbol{\upvarphi}=[\varphi_1,\varphi_2,...,\varphi_N]$. 
$\varphi_i \triangleq \arctan\big(\frac{U^*_y}{U^*_x}\big)$ is the direction of the collective potential field at a point $(x,y)$.
Here, $U^*_x=\frac{\partial U^*}{\partial  x}$ and $U^*_y=\frac{\partial U^*}{\partial y}$.
In (\ref{equ:deployment_vel_w}), $k_u>0$ and $k_w>0$ are control gains, $\textbf{u}=[u_1, u_2,..., u_N]$ and $\textbf{w}=[w_1, w_2,..., w_N]$ are the input linear and angular velocity vectors of all the robots in searching state $\mathcal{N}_s$, respectively.  $\dot{\varphi}_i$, the time derivative of $\varphi_i$, is shown as below considering the nonlinear model (\ref{equ:kine_1}): \begin{equation*}
\begin{aligned}
\dot{\varphi}_i = \frac{1}{{U^*_x}^2 + {U^*_y}^2}\bigg(\bigg(\frac{\partial U^*_y}{\partial x}\cos \theta_i + \frac{\partial U^*_y}{\partial y}\sin \theta_i \bigg) U^*_x - \\ 
\bigg(\frac{\partial U^*_x}{\partial x}\cos \theta_i + \frac{\partial U^*_x}{\partial y}\sin \theta_i \bigg) U^*_y \bigg)u_i
\end{aligned}
\end{equation*}
where $u_i$ is provided in (\ref{equ:deployment_vel_u}). Here, the singularity happens when ${U^*_x}^2 + {U^*_y}^2 = 0 \Leftrightarrow U^*_x = 0 \land U^*_y = 0$. The singularity can be reached only if robot $R_i$ is located at the local minimum of the potential field, and $R_i$ will be stuck. However, we conclude in Proposition~\ref{prop:intersection} that such local minimum does not exist if applying our proposed hybrid control scheme.

Noticeably, the input velocity $\|u_j\|$ for the searching state does not need to be bounded by the maximum speed in (\ref{equ:v_max}), because the robot has not yet started coverage operation. But in analysis, the velocity is made to be bounded by $v_{max}$ just to simplify the analysis process. Meanwhile, the robot acceleration is bounded mechanically according to (\ref{equ:a_max}) for searching state.

The linear control input (\ref{equ:deployment_vel_u}) is to motivate the robot $R_i$ to approach the goal $q_{g,i}$, while the angular control (\ref{equ:deployment_vel_w}) is to align the orientation of the robot to the direction of the collective potential field $U^*$, which is $\varphi_i$. The control law (\ref{equ:deployment_vel}) makes the robot with a unicycle model converge to the goal asymptotically from almost all initial conditions; its stability is proved in {\it Proposition 1} of \cite{valbuena2012hybrid}.
Considering that the robot may not arrive exactly at the designated goal position, or an oscillation may happen when the robot passes the goal position due to disturbance or inertia, we propose a dead zone value $\epsilon$ to terminate the searching state when $\|\textbf{q}_{g,j}-\textbf{q}_j\|<\epsilon$.

The reason why the repulsive potential component (\ref{equ:repulsive_APF_dep}) does not include those robots in rendezvous state $\mathcal{N}_r$ or searching state $\mathcal{N}_s$ is because the trajectory of the robot in boundary searching will not intersect with robots in rendezvous or the same searching state. This statement can be proved by contradiction in Proposition~\ref{prop:intersection}.

\begin{proposition}
\label{prop:intersection}
For a specific spill in the workspace, no intersection exists between the trajectories of any two robots if neither one is working under boundary tracking state. 
\end{proposition}

\begin{proof}
This proposition is proved by contradiction. Without loss of generality, we assume that there is one intersection between the trajectories of robots $R_i \in\mathcal{N}_s$ and $R_j \in \mathcal{N}_r$, while $R_i$ is starting from its current positions $\textbf{q}_{i}$ and moving toward it goal position $\textbf{q}_{g,i}$, and $R_j$ is moving from its current position $\textbf{q}_j$ toward to its rendezvous point $\textbf{q}_{rend,j}$, respectively. The intersection is denoted as $O$, shown in Fig.~\ref{fig:intersection} (left). If connected with line segments, the four points form a quadrilateral. We can show that $\overline{O\textbf{q}_i}+\overline{O\textbf{q}_{g,i}}>\overline{\textbf{q}_i\textbf{q}_{rend,j}}$ and $\overline{O\textbf{q}_{j}}+\overline{O\textbf{q}_{rend,j}}>\overline{\textbf{q}_j\textbf{q}_{g,i}}$, thus we show $\overline{\textbf{q}_i\textbf{q}_{g,i}}+\overline{\textbf{q}_{j}\textbf{q}_{rend,j}}>\overline{\textbf{q}_{i}\textbf{q}_{rend,j}}+\overline{\textbf{q}_j\textbf{q}_{g,i}}$. If this conclusion stands, the better rendezvous point for $R_j$ will be $\textbf{q}_{g,i}$, which is on the boundary. Moreover, $R_j$ will not remain in rendezvous state for long as it can detect the boundary (i.e., point $\textbf{q}_{g,i}$). Meanwhile, since $R_i$ does not rendezvous to point $\textbf{q}_{rend,j}$, they do not reside in the same hierarchical rendezvous spanning tree which is described in \cite{parasuraman2018multipoint}, meaning that they do not have an intersected trajectories for being in different rendezvous spanning trees. Similarly, it can be proved that two robots working under searching state will not collide by replacing $\textbf{q}_{rend,j}$ with $\textbf{q}_{g,j}$.    Those two robots can have less total displacement if their goal positions are swapped, which obviously violates Algorithm~\ref{alg:B_searching}. Finally, no collision will happen between robots even if they are within the same rendezvous tree, according to \cite{luo2019multi}. Conclusively, there is no intersection of the trajectories of any two robots if none of them is in boundary tracking state. 

\begin{figure}[t]
\centering
\includegraphics[width=0.7\columnwidth]{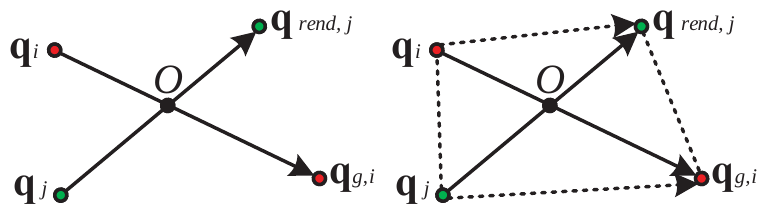}
\caption{A representative figure showing an intersection between trajectories of two robots $R_i$ and $R_j$.}
\label{fig:intersection}
\end{figure}
\end{proof}

Deadlock equilibrium is a common issue that hinders robot movement in a multi-robot situation. Although robots working in different states can avoid collision by eliminating a saddle point, as discussed above, they may still face a deadlock issue for unicycle models. Specifically, a robot with a unicycle model has to steer from the initial orientation while approaching to the goal position, whose trajectories can be seen in Fig.~\ref{fig:case_1}. This steering may exert conflicting forces upon adjacent robots towards opposite directions and hence causes a deadlock equilibrium.
To eliminate such deadlock, we can increase the scope range $d_0$ in (\ref{equ:repulsive_APF_dep}) to allow more room for steering, or stop robots before entering a computed collision zone and resume motion once the zone is clear \cite{soltero2011collision}. Alternatively, strategies based on robot trajectory re-planning such as \cite{qutub1997solve} or reciprocal velocity obstacle method \cite{van2008reciprocal} can also be utilized.

For an agent $R_i$ with a field of view (FOV) of span $(-\phi, \phi)$, an angular velocity control law needs to be determined such that the agent is aligned with the direction of the boundary, especially when the robot is already on the boundary but needs re-orientation, such as the cases shown in Fig.~\ref{fig:FOV_peri_1} and Fig.~\ref{fig:FOV_peri_2}. If the boundary does not initially reside in the FOV, this angular velocity control law is applied and the robot starts searching and aligning to the boundary. The searching state continues until the boundary is maintained within the field of tracking (FOT), a small view angle denoted $\pm\varepsilon_{\phi}$ and located within the FOV of robot, as shown in Fig.~\ref{fig:FOV_overall}. 
The FOT is an axial dead zone designed for angular control that avoids frequent oscillation when robot is advancing.
\begin{figure}[t]
\centering
\begin{tabular}{ccc}
\subfigure[FOV/FOT setting]{\label{fig:FOV_overall}
\includegraphics[width=0.135\textwidth]{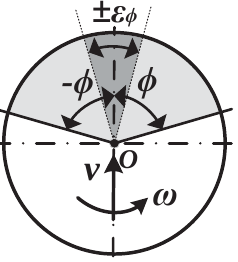}} &
\subfigure[$\partial\mathcal{S}$ search case 1]{\label{fig:FOV_peri_1}
\includegraphics[width=0.14\textwidth]{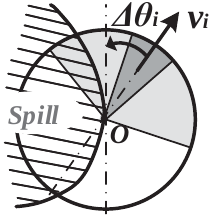}} &
\subfigure[$\partial\mathcal{S}$ search case 2]{\label{fig:FOV_peri_2}
\includegraphics[width=0.13\textwidth]{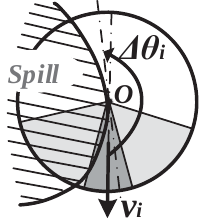}}
\end{tabular}
\caption{(a) Definition of agent's field of view (FOV), which is a section spanning $[-\phi, \phi]$ and highlighted with light shadow, and field of tracking (FOT), a section spanning $[-\varepsilon_\phi, \varepsilon_\phi]$ being highlighted with dark shadow. (b) and (c) are two common spill boundary $\partial\mathcal{S}$ searching cases with angular changes produced by control laws and algorithms.}
\end{figure}
Algorithm~\ref{alg:B_searching} presents the control algorithm of boundary searching in detail. Before boundary searching starts, few system internal states that may affect Algorithm~\ref{alg:B_searching} are decided beforehand as initialization in Algorithm~\ref{alg:B_searching_init}.

Persistent tracking of spill boundary $\partial\mathcal{S}$ is accomplished by Algorithm~\ref{alg:B_tracking} and lasts until the coverage work is done. If a robot $R_i$ worked in tracking boundary but suddenly misses the boundary from its vision possibly due to an overshoot or large disturbance, it needs to switch back to the searching state and detect the boundary again, as shown in Fig.~\ref{fig:state_transition}. Note that for a drastically changing environment where the spill boundary may escape from the vicinity of the robot and cannot be detected by vision sensors any longer, the robot team needs to enter rendezvous state and move to the boundary again. Under certain circumstance, the robot can remain in tracking state despite of a dynamic spill. More details can be found in Sec.~\ref{sec:dynamic}.

A special case exists that the initial boundary may be shrunk by robots rooted at rendezvous $D_i$ beyond it can be sensed by robots from $D_j, (j \neq i)$; thus, rendezvous robot $D_j$ loses its vicinity of the boundary. To circumvent this issue, we design two modes for a robot in tracking, namely {\it normal mode} and {\it idle mode}. The robot performs boundary shrink control only under normal mode, while tracking without shrinking the boundary in idle mode. One robot $R_i$ switches to idle mode when it resides within an effective range of a static rendezvous robot $D_j$ (with child robots associated), i.e., $\|\textbf{q}_i-\textbf{q}_{D_j}\| \leq \alpha \cdot r_{\mathcal{C}}, \alpha \in (0,1)$. Therefore, the boundary detected by $D_j$ will remain until rendezvous completes. This fact is as well reflected in Line 7$\sim$10 of Algorithm~\ref{alg:B_tracking}.

\begin{algorithm}[t]
%\scriptsize
\DontPrintSemicolon
Initialization {\it /*All flags clear to 0 */}\;
\Repeat{robot $R_i$ rotates in the same place for $2\pi$}{
    \If{$R_i$ detects itself sitting on the spill boundary $\partial\mathcal{S}$}{
        Set $on\_the\_boundary\_flag \leftarrow 1$\;
        \If{boundary $\partial\mathcal{S}$ resides in FOT}{
            Set $tracking\_enable\_flag \leftarrow 1$\;
            \Break
        }
    }
    \If{$\partial\mathcal{S}$ can be detected and is within the FOV}{
        Set $boundary\_within\_FOV\_flag \leftarrow 1$\;
        \Break
    }
    \If{no boundary $\partial\mathcal{S}$ is detected}{
        Set $no\_boundary\_detected\_flag \leftarrow 1$\;
        Set $u_i=0,\, w_i=0$\;
        \End \, {\it /*Terminate this program, as no spill is detected within vision range $r_{\mathcal{A}}$*/}\;
    }
}
\caption{Boundary $\partial\mathcal{S}$ Searching Initialization} 
\label{alg:B_searching_init}
\end{algorithm}

\begin{algorithm}[!h]
%\scriptsize
\Repeat{$tracking\_enable\_flag = 1$}{
    \If{$on\_the\_boundary\_flag = 0$}{
        \If{$boundary\_within\_FOV\_flag = 1$}{
            \Repeat{$on\_the\_boundary\_flag = 1$}{
                Identify the nearest point $\textbf{q}_{g,i}$ on $\partial\mathcal{S}$  to the robot\;
                Apply control law (\ref{equ:deployment_vel}) to drive robot to $\textbf{q}_{g,i}$\;
            }
        }
    }
    \ElseIf{$on\_the\_boundary\_flag = 1$}{
        The robot linear velocity control input $u_i = 0$\;
        \If{$boundary\_within\_FOV\_flag = 1$}{
            \If{spill is on the right hand side of $R_i$}{
                $R_i$ rotates counterclockwise with respect to its centroid for a degree of $\phi_{temp}=\omega_{max}\cdot T_c$ such that $\partial\mathcal{S}$ is outside the FOV\;
                Clear $boundary\_within\_FOV\_flag \leftarrow 0$\;
            }
            \Else{
                \Repeat{$\partial\mathcal{S}$ is within the FOT}{
                    Apply control laws (\ref{equ:angular_law_discrete}) and (\ref{equ:angular_law}) for direction alignment \,
                    {\it /* Since $\partial\mathcal{S}$ is within FOV and spill is on the left hand side of $R_i$*/\;}}
                    Set $tracking\_enable\_flag \leftarrow 1$
                }
        }
        \ElseIf{$boundary\_within\_FOV\_flag = 0$}{
            \Repeat{$\partial\mathcal{S}$ is within the FOV}{
                $R_i$ begins rotating counterclockwise with respect to its centroid for a degree of $\phi_{temp}$ \;
            }
            Set $boundary\_within\_FOV\_flag \leftarrow 1$\;
        }
    }
}
\Return {Switch to the Tracking state \, {\it /* End the program, and switch to the Tracking state in Algorithm \ref{alg:B_tracking} */\;}}
\caption{Boundary $\partial\mathcal{S}$ Searching Algorithm} 
\label{alg:B_searching}
\end{algorithm}

\begin{algorithm}[t]
%\scriptsize
\DontPrintSemicolon
The robot $R_i$ is located on the boundary of an associated spill\;
The velocity control input of $R_i$ follows constraints (\ref{equ:v_max}) and (\ref{equ:a_max})\;
\Repeat{forever}{
    \If {$\partial\mathcal{S}$ is within the FOV}{
        \If {$\partial\mathcal{S}$ is within the FOT}{
            Apply control law (\ref{equ:tracking_vel}) and (\ref{equ:u_i_final})\;        
            \If {$\|\textbf{q}_i-\textbf{q}_{D_j}\| \leq \alpha \cdot r_{\mathcal{C}}, D_j \in D$}{
                Enable {\it idle mode}\;
            }
            \Else{Enable {\it normal mode}\;
            }
        }
        \Else {Apply control laws (\ref{equ:angular_law_discrete}) and (\ref{equ:angular_law})\;}  
    }
    \Else{
        \Return Switch to the Searching state \quad{\it /* End the program, tracking failed, switch to the Searching state in Algorithm~\ref{alg:B_searching}*/}
    }
}
\caption{Boundary $\partial\mathcal{S}$ Tracking Algorithm} 
\label{alg:B_tracking}
\end{algorithm}

Special consideration should be given in determining the $\phi_{cm}$ in Algorithm~\ref{alg:B_searching}. On one hand, it should be as large as possible in order to enlarge the step size, making rotation complete in fewer operation cycles. But on the other hand, it should not be so large that the robot misses the FOV in one rotation. Evidence for this is found in Algorithm~\ref{alg:B_searching}, where the robot has to stop moving and indicate `search fails' if failing to detect the boundary in one full rotation. Therefore, we conclude that
$\phi_{temp}=\omega_{max}\cdot T_c \leq 2\phi$. Thus, $\omega_i$ is bounded by:
\begin{equation}
\omega_{max} = \frac{2\phi}{T_c}.
\label{equ:omega_max}
\end{equation}

If the boundary $\partial\mathcal{S}$ is in the FOV but not in the FOT, with the heading angle of agent being $\theta_i$, the angle error with respect to the desired angle $\theta_d(=0)$ is defined as $\theta_{ei}=\theta_d-\theta_i$. 
$\theta_d$ is the angle of spill boundary, which is detected using vision sensors of robots. 
\cite{Salamanca} presented a method to transform an input vision image into a bird-eye view image. Then, using the bird-eye view image, we can estimate $\theta_d$. How to estimate $\theta_d$ is beyond the scope of this paper.
One can apply the following proportional-derivative (PD) control law to align the forward direction of the robot with the desired angle $\theta_d$:
\begin{equation}
{w}_i = \begin{cases} 
-K_p \cdot \theta_{ei} - K_d\cdot \dot{\theta}_{ei}, \\
\qquad\qquad \text{if $\partial\mathcal{S}$ appears in sector $[-\phi, -\varepsilon_\phi)\cup(\varepsilon_\phi, \phi]$,} \\
0, \qquad\quad \text{if $\partial\mathcal{S}$ appears in sector $[-\varepsilon_\phi, \varepsilon_\phi]$,}
\end{cases} 
\label{equ:angular_law_continuous}
\end{equation}
where $K_p$ and $K_d$ are proportional and derivative
gains, respectively.

Let $T$ denote the current time when the control is applied.
Considering the fact that the robot control is discrete in time, the following discrete control law is adopted instead of (\ref{equ:angular_law_continuous}):
\begin{equation}
{w}_i = \begin{cases} 
-K_p \cdot \theta_{ei}(T) - \frac{K_d}{T_c}\big(\theta_{ei}(T)-\theta_{ei}(T-T_c)\big), \\
\qquad\qquad \text{if $\partial\mathcal{S}$ appears in sector $[-\phi, -\varepsilon_\phi)\cup(\varepsilon_\phi, \phi]$,} \\
0, \qquad\quad
\text{if $\partial\mathcal{S}$ appears in sector $[-\varepsilon_\phi, \varepsilon_\phi]$.}
\end{cases} 
\label{equ:angular_law_discrete}
\end{equation}

Finally, by considering the upper limit of angular speed shown in (\ref{equ:omega_max}), we update the controller to be:
\begin{equation}
\|w_i\| = \begin{cases} 
\|K_p \cdot \theta_{ei}(T) + \frac{K_d}{T_c}\big(\theta_{ei}(T)-\theta_{ei}(T-T_c)\big)\|, \\ 
\qquad\qquad\qquad\qquad\qquad\qquad\qquad \text{if $\|w_i\| \leq \omega_{max}$,} \\
\omega_{max}, \qquad\qquad\qquad\qquad\qquad\quad
\text{otherwise.}
\end{cases} 
\label{equ:angular_law}
\end{equation}

\iffalse
\begin{equation}
w_i = \begin{cases} 
\frac{|\Delta \theta_i|}{\phi}\omega_{max}, & \text{if $\partial\mathcal{S}$ appears in $R_i$'s $[-\phi, -\varepsilon_\phi]$ sector,} \\
0, & \text{if $\partial\mathcal{S}$ appears in $R_i$'s $[-\varepsilon_\phi, \varepsilon_\phi]$ sector,}  \\
-\frac{|\Delta \theta_i|}{\phi}\omega_{max}, & \text{if $\partial\mathcal{S}$ appears in $R_i$'s $[\varepsilon_\phi, \phi]$ sector}
\end{cases} 
\label{equ:angular_law}
\end{equation}
where $\Delta \theta_i$ stands for the angle between current forward direction and tangent of current robot waypoint on the boundary. The procedure of this case is shown in Fig.~\ref{fig:FOV_peri_1}.
\fi

Specifically, if the boundary $\partial\mathcal{S}$ is outside the FOV, the proposed control laws are able to rotate the robot until $\partial\mathcal{S}$ falls into the FOV. However, if the $\partial\mathcal{S}$ resides in FOV but the spill is on the right hand side of the robot, that means the robot is heading the wrong direction. Hence, the robot has to rotate counterclockwise until it catches sight of $\partial\mathcal{S}$ again, as illustrated in Fig.~\ref{fig:FOV_peri_2}.
This procedure is presented at Line 10$\sim$13 of Algorithm~\ref{alg:B_searching}.
Note the rotation process outside FOV is open-loop based and aims for no specific goal. The proposed $\omega_{max}$ guarantees the minimum possible steps to finish searching. When the $\partial\mathcal{S}$ is once again within the FOV, control laws (\ref{equ:angular_law_discrete}) and (\ref{equ:angular_law}) can be applied. Both (\ref{equ:angular_law_discrete}) and (\ref{equ:angular_law}) are based on a vision sensor feedback available at low cost in terms of implementation. If more accurate and consistent control is needed for more complex environments, one can refer to the visual servoing methods proposed in \cite{sattar2005visual,lots20012d}. 

\subsection{Boundary Tracking Controller and Convergence Proof}
\label{sec:APF}
When a group of robots is performing boundary shrinking for spill coverage along the boundary, a proper control method needs to be proposed to avoid collision between robots but still motivate them to move forward. Such a controller based on the APF method is provided and discussed in this section.

The robots are re-labeled after the completion of allocation. For a set of sequential indices, they are assumed to be associated with a specific spill. Two indices of robots are sequential means they are adjacent, such as $R_{i+1}$ and $R_i$. Such re-labeling is considered as well in the analysis of  Sec.~\ref{sec:stability}. 

\begin{figure}[t]
\centering
\includegraphics[width=0.99\columnwidth]{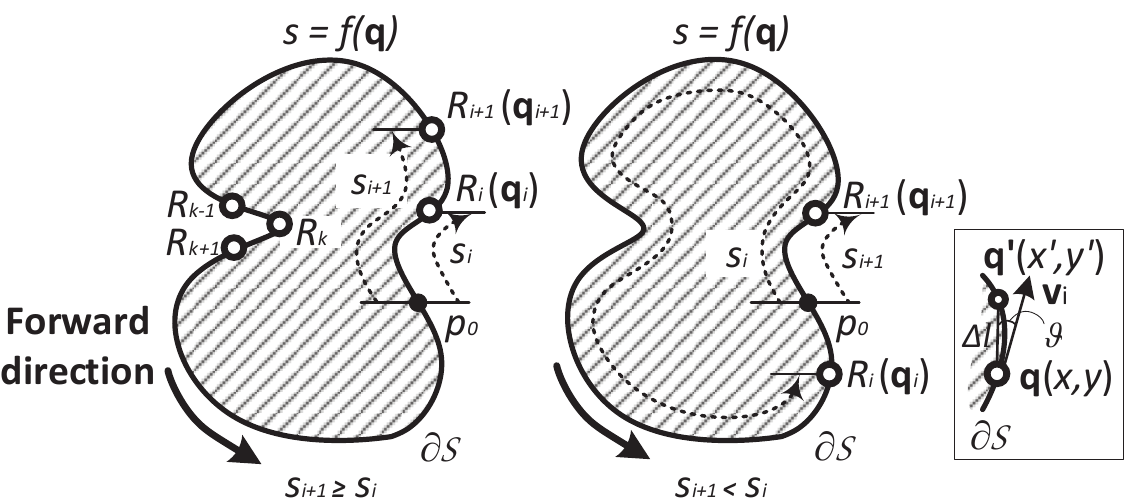}
\caption{The left figure shows two subgroups of adjacent robots, one subgroup consisting of $R_{k-1}$, $R_{k}$, and $R_{k+1}$, and the other consisting of $R_{i}$ and $R_{i+1}$. Both the left and middle figures illustrate the definition of $s_i$ and $s_{i+1}$ for the trailing robot $R_{i}$ and its leading robot $R_{i+1}$, respectively. The right figure in the box shows the symbols and their geometrical relationship when a robot is maneuvering along the boundary $\partial\mathcal{S}$.}
\label{fig:l_s_define}
\end{figure}

Since workspace $\mathcal{W}$ is assumed to be obstacle-free, attractive potential $U_{att}$ exerted on $R_i$ by its leading robot $R_{i+1}$ is sufficient to drive $R_i$. When $R_i$ is tracking the boundary $\partial\mathcal{S}$, it can measure the arc length in boundary rather than line distance between itself and other agents within its FOV. 
It uses arc length in boundary because a closer line distance between two robots does not mean that they are adjacent if tracking a serpentine boundary.
For instance, in Fig.~\ref{fig:l_s_define}, the adjacent robot of $R_{k-1}$ is $R_{k}$ rather than $R_{k+1}$, although $R_{k+1}$ is closer to $R_{k-1}$ in a line distance.
Practically, such measurement can be performed with LIDAR, high resolution laser range-finders, or vision sensors equipped with imaging techniques proposed in \cite{gowal2011two,mitiche2014computer}. Suppose that a function $s = f(\textbf{q})$ can represent the spill boundary $\partial\mathcal{S}$, where $s \in [0, \|\partial\mathcal{S}\|]$ demonstrates the arc length between a reference point and an agent located at $\textbf{q}$ in the counterclockwise direction. 
Here, $\|\partial\mathcal{S}\|$ indicates the total length of $\partial\mathcal{S}$.
Note that the definition of $s$ and a starting point $p_0$ are obtained from a global view and used only for stability and convergence analyses.
We can show a pair of adjacent robots $R_i$ and $R_{i+1}$ in Fig.~\ref{fig:l_s_define}, where  $R_{i+1}$ is the leading robot of $R_{i}$. Once $R_{i+1}$ falls in the FOV of its following robot $R_{i}$, arc length distance between these two adjacent robots, namely $l_i=\|\textbf{q}_i, \textbf{q}_{i+1}\|_{\partial\mathcal{S}}$, can be computed by
\begin{equation}
l_i = \begin{cases}
s_{i+1}-s_i, & \text{if}\, s_{i+1} \geq s_i, \\
s_{i+1}-s_i+\|\partial\mathcal{S}\|, & \text{if}\, s_{i+1} < s_i.
\end{cases}
\label{equ:l_i}
\end{equation}
$s_i$ is defined as the $s$ in the view of $R_i$, and $p_0$ in Fig.~\ref{fig:l_s_define} is the reference point for $s$. Recall that $r_{\mathcal{A}}$ is the vision sensor range and is illustrated in Fig.~\ref{fig:seaswarm}. $s_{i+1}\leq s_i$ can happen depending on where the reference point is, which is shown in Fig.~\ref{fig:l_s_define}. 
If the leading robot of $R_i$ is beyond the radius $r_{\mathcal{A}}$ of its FOV, (\ref{equ:l_i}) does not apply. Therefore it defines a virtual distance $l_i^*$ which includes both situations as below
\begin{equation}
l_i^* = \begin{cases}
l_i, & \text{if}\, l_i \leq r_{\mathcal{A}}, \\
r_{\mathcal{A}}, & \text{if}\, l_i > r_{\mathcal{A}} \, \text{or $l_i$ is unknown}.
\end{cases}
\label{equ:l_i_star}
\end{equation}

The attractive potential function is defined as 
\begin{equation}
U_{att} = \frac{1}{2}\xi_3 {l_i^*}^2
\label{equ:U_att_1}
\end{equation}
where $\xi_3$ is a positive scaling factor. 
The velocity $\textbf{v}_i$ should be in the direction of negative gradient of $U_{att}$ with respect to $s_i$ such that 
\begin{equation}
\textbf{v}_i = -\nabla_{s_i} U_{att} = \xi_3l_i^*.
\label{equ:U_att_1_v}
\end{equation}
However, extra factors need to be considered when applying (\ref{equ:U_att_1_v}) as the control input.

\subsubsection{Consider the maximum linear speed}
It is worth noticing that the robot cannot exceed the maximum linear speed $v_{max}$. Thus, we should have 
\begin{equation}
\textbf{v}_i = \begin{cases}
\xi_3l_i^*, & \text{if}\, \|\textbf{v}_i\| \leq v_{max}, \\
\frac{\xi_3l_i^* v_{max}}{\|\xi_3l_i^*\|}, & \text{otherwise}.
\end{cases}
\label{equ:u_i_final}
\end{equation}

\subsubsection{Control design}
\label{sec:control_design}
Controllers similar to (\ref{equ:u_i_final}) are developed in the works such as \cite{zhang2013spill}; however, they do not consider or fully discuss the control design under unicycle model (\ref{equ:kine_1}). Here, we provide a strategy inspired by \cite{pickem2015gritsbot} to determine the velocity control input tailored for unicycle model. Note that the robot acceleration is bounded mechanically according to (\ref{equ:a_max}) for boundary searching and tracking states.

Let $(x,y)$ present the current position of $R_i$.
Assuming a point $(x',y')$ lying on $\partial\mathcal{S}$ and in front of the robot $R_i$ with a small length offset of $\Delta l$, we have its global coordinate as $x' = x + \Delta l\cos (\vartheta+
\theta_i)$ and $y' = y + \Delta l\sin (\vartheta+
\theta_i)$, having $\vartheta$ as the angle between orientations of the robot and the potential field $U_{att}(x,y)$. Recall that $\theta_i$ is the orientation direction of the robot $R_i$.
The aforementioned symbols and their geometrical relationship can be found in Fig.~\ref{fig:l_s_define} (in the box), where the orientation of the potential field is represented as $\textbf{v}_i$. 
Practically, $(x',y')$ can be sensed by local vision sensors of $R_i$.
Let $\textbf{Q} = [u_i,w_i]^T$ be the velocity control input and $\textbf{v}_i = [v_x^i,v_y^i]^T$ in (\ref{equ:u_i_final}), the velocity becomes as follow:
\begin{equation}
\begin{bmatrix} \dot{x'} \\ \dot{y'} \end{bmatrix} = \begin{bmatrix}  v_x^i \\ v_y^i \end{bmatrix} = \begin{bmatrix} \cos{\vartheta} & -\sin{\vartheta} \\ \sin{\vartheta}  & \cos{\vartheta} \end{bmatrix} \begin{bmatrix} 1 & 0 \\ 0  & \Delta l \end{bmatrix} \begin{bmatrix} u_i \\ w_i \end{bmatrix} = \textbf{G}(\vartheta,\Delta l)\textbf{Q}.
\label{equ:control_design}
\end{equation}
Then the control input can be obtained by $\textbf{Q} = \textbf{G}^{-1}\textbf{v}_i$, i.e.,
\begin{subequations}
\label{equ:tracking_vel}
\begin{align}
u_i &= v_x^i\cos\vartheta + v_y^i\sin\vartheta, \label{equ:tracking_vel_u}\\
w_i &= \frac{1}{\Delta l}(-v_x\sin\vartheta + v_y\cos\vartheta).  \label{equ:tracking_vel_w}
\end{align}
\end{subequations}
Note that the controller (\ref{equ:tracking_vel}) is bounded using  (\ref{equ:u_i_final}), where the the velocity set $\textbf{v}_i = [v_x^i,v_y^i]^T$ in (\ref{equ:tracking_vel}) is derived from (\ref{equ:u_i_final}).
In the case of coarse measurement or irregular spill shapes such that (\ref{equ:l_i}) cannot be obtained with precision, cooperative Kalman filters presented in \cite{zhang2010cooperative} can be utilized to estimate the boundary curvature and determine the control law, when a noisy scalar field is built for the spill.

\subsubsection{Speed convergence}
\label{sec:speed_conv}
When $l_i^* \rightarrow 0$, we need to show the convergence of linear speed. By differentiating $U_{att}(\textbf{q}_i)$ in (\ref{equ:U_att_1}) with respect to time $t$ and considering (\ref{equ:u_i_final}), we have 
\begin{equation}
\dot{U}_{att} = -\xi_3 \dot{l}_i^* = -\xi_3^2l_i^* = -2\xi_3{U}_{att}.
\label{equ:U_att_1_diff}
\end{equation}
From (\ref{equ:U_att_1_diff}), it shows 
\begin{equation}
{U}_{att} = {U}_{att}(0)e^{-2\xi_3^2t}.
\label{equ:U_att_1_res}
\end{equation}
Obviously, the maximum potential $U_{att}=\frac{1}{2}\xi_3 r_{\mathcal{A}}^2$ exists if the distance of boundary between robot $R_i$ and its leader $R_{i+1}$ are far enough. Therefore, it is preferable if $R_{i+1}$ travels beyond the FOV of $R_i$. Moreover, the greater $\xi_3$ is, the faster $U_{att}$ and $l_i^*$ converge to 0. One last significant remark is, from control laws (\ref{equ:u_i_final}) and (\ref{equ:angular_law}), we can see that one trailing robot $R_i$ cannot wrap up its leading robot $R_{i+1}$, because along the boundary, the attractive potential decreases to {\it zero} as $R_i$ is approaching $R_{i+1}$, thus linear velocity input also decreases to {\it zero}.

\subsubsection{Robot convergence after boundary shrink}
\label{sec:consensus}
When coverage is coming to an end and boundary shrinks to almost a point, all the robots associated with a specific spill should then form a loop and converge around a place, as long as the spill remains convex during coverage. One might ask whether the robots will reduce their speed while moving to the endpoint and eventually stop when they have arrived. A positive conclusion is given in the discussion in Sec.~\ref{sec:speed_conv}, which shows that the linear speed $\|u_i\|$ for every agent converges to {\it zero} when $l_i^*$ approaches {\it zero}. The whole process stops when the robot cannot find any more spill, or, more specifically, when robot FOV still misses boundary $\partial\mathcal{S}$ after a full rotation under searching state. This implies that the spill have been fully cleaned. Practically, a complete coverage can be achieved when the conveyor belts of robots, as shown in Fig.~\ref{fig:seaswarm}, are crowding toward each other. The dimensions of the robot may affect the convergence result, since the robots' dimensions prevent them from getting closer to each other. Nevertheless, this issue can be tackled by improving robot design, for instance, a retractable or flexible  conveyor belt can be adopted for robots in Fig.~\ref{fig:seaswarm}. The stability of multi-robot coordination control in coverage is detailed in Sec.~\ref{sec:stability}.
The information about nonconvex spills is provided in Sec.~\ref{sec:nonconvex_spill}.

\subsection{Stability Analysis}
\label{sec:stability}
We investigate the stability of robot coordination in coverage control. At a certain time $t=T_+$ after long operation, we assume all robots deployed are in the tracking state, and each of them has $l_i\leq r_{\mathcal{A}}$, meaning that it has a leading robot within its view and resides within the view of a trailing robot. Defining the $n_j$ number of robots allocated for spill $\mathcal{S}_j$ as $R_1^{j}, R_2^{j},..., R_{n_j}^{j}$, the boundarys for spills $[\|\partial\mathcal{S}_1\|,\|\partial\mathcal{S}_2\|,...,\|\partial\mathcal{S}_M\|]_{t=T_+}$, and $\textbf{l}^{*} = [l^{*}_1,l^{*}_2,...,l^{*}_{N}]^T_{t=T_+}$ at the current time, the system collective potential energy is provided by
\begin{equation}
V(\textbf{s})=V(\textbf{l}^*)=\sum^M_{j=1}\bigg(\sum^{n_j}_{i=1} U_{att}(i) \bigg).
\label{equ:coordi_stability}
\end{equation}
Here, $\textbf{s}=[s_1,s_2,...,s_N]^T$, and $U_{att}(i)$ is the attractive potential shown in (\ref{equ:U_att_1}). 

We define a sparse matrix $Z = [z_{ij}]_{M \times N}$ to represent the association of each robot to a specific spill, where $z_{ij} = 1$ means robot $R_i$ is associated to spill $\mathcal{S}_j$. The matrix $Z$ also satisfies constraint (\ref{equ:allocation_const}).
\begin{equation} \label{equ:allocation_const}
\sum_{i=1}^{N} z_{ij} \geq 1, \quad \sum_{j=1}^{M} z_{ij} = 1, \quad \forall j\in \{1, 2,..., M\},
\end{equation}
Obviously, the collective potential (\ref{equ:coordi_stability}) has to satisfy a boundary constraint as shown below
\begin{equation}
Z\cdot \textbf{l}^{*} = [\|\partial\mathcal{S}_1\|,\|\partial\mathcal{S}_2\|,...,\|\partial\mathcal{S}_M\|]^T_{t=T_+}.
\label{equ:coordi_stability_condition}
\end{equation}

Using (\ref{equ:coordi_stability}) and the result of Sec.~\ref{sec:speed_conv}, we have the time derivative of V(\textbf{s}) as
\begin{equation*}
\begin{aligned}
\dot{V}(\textbf{s}) &= \frac{dV_1}{ds_1}\cdot \frac{ds_1}{dt} + \frac{dV_2}{ds_2}\cdot \frac{ds_2}{dt} +...+ \frac{dV_N}{ds_N}\cdot \frac{ds_N}{dt} \\
&=-\xi_3 \textbf{l}^{*}\cdot \textbf{v} \leq 0
\end{aligned}
\end{equation*}
where $\textbf{v}=[v_1, v_2,..., v_N]^T$. Thus, we show that the Lyapunov function $V(\textbf{s})$ is negative semi-definite, and $\dot{V}=0\,\, iff\, \textbf{l}^{*}=0$, i.e., $s_i = s_{i+1}$ for any two adjacent robots tracking the same spill and $\|\partial\mathcal{S}_i\|=0$, according to (\ref{equ:l_i}), (\ref{equ:l_i_star}), and (\ref{equ:coordi_stability_condition}). 
Meanwhile, the velocity $v_i$ is bounded by (\ref{equ:u_i_final}) and it decays with $l_i^*$.
Therefore, according to LaSalle's invariance principle \cite{lasalle1976stability}, those robots deployed to a specific spill will asymptotically converge to a place when the spill boundary shrinks to a point, i.e., the spill is removed.

\subsection{Dynamic Spill Applicability}
\label{sec:dynamic}

\begin{figure}
    \centering
    \includegraphics[width=0.6\columnwidth,trim={0 0 0 0},clip]{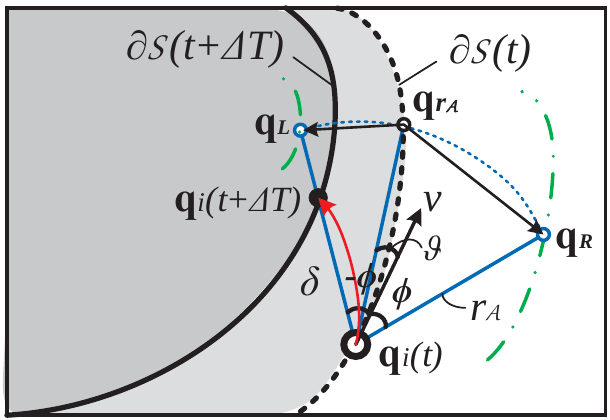}
    \caption{The trajectory (red line) of robot $R_i$ working in the boundary tracking state when the boundary is changing. The dashed black curve is the boundary at time $t$ while the solid black curve is the boundary in $t+\Delta T$. The blue lines with a span $(-\phi, \phi)$ indicates the field of view (FOV) of robot. The dash-dotted lines in green represent the critical location of boundary escaping the robot vision range from left and right side, respectively.}
    \label{fig:dynamic_boundary}
\end{figure}

In real cases, the boundary of a spill cannot be always static during the operation. In this section, we prove that our proposed boundary shrink control strategy can also handle a case when the spill boundary is moving. 

The multi-robot rendezvous algorithm deals with moving boundary intrinsically by motivating the rendezvous robot $D_m$ to move according to the boundary. Therefore, the child robots will eventually gather in the vicinity of the changing boundary. Such behavior that has a dynamic rendezvous point is named as robotic {\it herding} and validated in \cite{parasuraman2018multipoint}.

In boundary searching state, our solution can handle moving boundary case by simply remaining in this state and rerunning Algorithm~\ref{alg:B_searching}, provided that the boundary does not move faster than robot maximal speed $v_{max}$ in the case that the boundary escapes from robot vision range $r_\mathcal{A}$ when the robot is chasing the goal position on the boundary. 

If boundary moving happens in the boundary tracking state, then our Algorithm~\ref{alg:B_tracking} can tackle this issue, even without switching to boundary searching state if the changing boundary appears inside the FOV, as shown in Fig.~\ref{fig:dynamic_boundary}. 
Since the changing boundary appears inside robot FOV, the trajectory of robot working in the boundary tracking state and a dynamic boundary is shown in the red line in Fig.~\ref{fig:dynamic_boundary}. In the same figure, robot's current position is $\textbf{q}_i(t)$, and the goal position on the dynamic boundary for the robot to track denotes $\textbf{q}_i(t+\Delta T)$, whose coordinate can be calculated as%by applying  Pythagorean theorem:
\begin{equation}
    \textbf{q}_i(t+\Delta T) = \begin{bmatrix}  x(t+\Delta T) \\ y(t+\Delta T) \end{bmatrix} = \begin{bmatrix} \delta\cos{(\phi+\theta_i)} + x(t) \\ \delta\sin{(\phi+\theta_i)} + y(t) \end{bmatrix},
\label{equ:pythagorean}
\end{equation}
where $[x(t), y(t)]^T$ is the coordinate of $\textbf{q}_i(t)$, and $\delta$ is the distance between $\textbf{q}_i(t)$ and $\textbf{q}_i(t+\Delta T)$ which can be determined by image processing techniques such as \cite{saxena2006learning}. 
Using control law (\ref{equ:deployment_vel}) by replacing $\textbf{q}_g$ in it with $\textbf{q}_i(t+\Delta T)$, the robot can track the dynamic boundary.

For a robot working under boundary tracking state, as shown in Fig.~\ref{fig:dynamic_boundary}, the spill boundary can escape the robot vision range from either left or right side if $\textbf{q}_{r_\mathcal{A}}$, the furthest point on the boundary that is within robot vision range, moves over point $\textbf{q}_{L}$ or $\textbf{q}_{R}$. Given this fact, we can decide the maximum speed for the moving spill, $v_{max}^{\partial\mathcal{S}}$, such that the boundary always resides in robot's FOV. $v_{max}^{\partial\mathcal{S}}$ can be determined as $v_{max}^{\partial\mathcal{S}} = \min\{\frac{\|\textbf{q}_{L}-\textbf{q}_{r_\mathcal{A}}\|}{T_c}, \frac{\|\textbf{q}_{R}-\textbf{q}_{r_\mathcal{A}}\|}{T_c}\} = \min\{\frac{(\phi-\vartheta)r_\mathcal{A}}{T_c}, \frac{(\phi+\vartheta)r_\mathcal{A}}{T_c}\}$. Here, recall that $T_c$ is the robot control system cycle time.

\section{Simulation Results}
\label{sec:experiments}

To evaluate the proposed boundary shrink control strategy, extensive experiments using simulation were conducted in Robotarium MATLAB platform \cite{pickem2017robotarium} which is an open source swarm robotics experiment testbed. Robotarium also supports hardware experiments with GRITSbot robots. 
Robotarium has a 2D arena of size $3\text{m}\times 3\text{m}$, which can be used to simulate a scale-down aquatic workspace such as lakes and oceans where spills are patched. In the experiments, four spills with different geometric features are presented, including two circles with different areas, one ellipse and one quadrilateral, as shown in Fig.~\ref{fig:simu_envrion_setting}. The diverse geometry of presented spills will help to validate the capability of the proposed method. In addition, Table~\ref{tab:spill_areas} shows the index and area of the spill patches, and $v_{max}$ in (\ref{equ:v_max}) is made equal to $0.01\,\text{m/s}$ to cope with this testbed.

\begin{figure}[t]
%\vspace{-0.2cm}
\centering
\includegraphics[width=1.2\columnwidth,trim={10cm 0 5cm 0},clip]{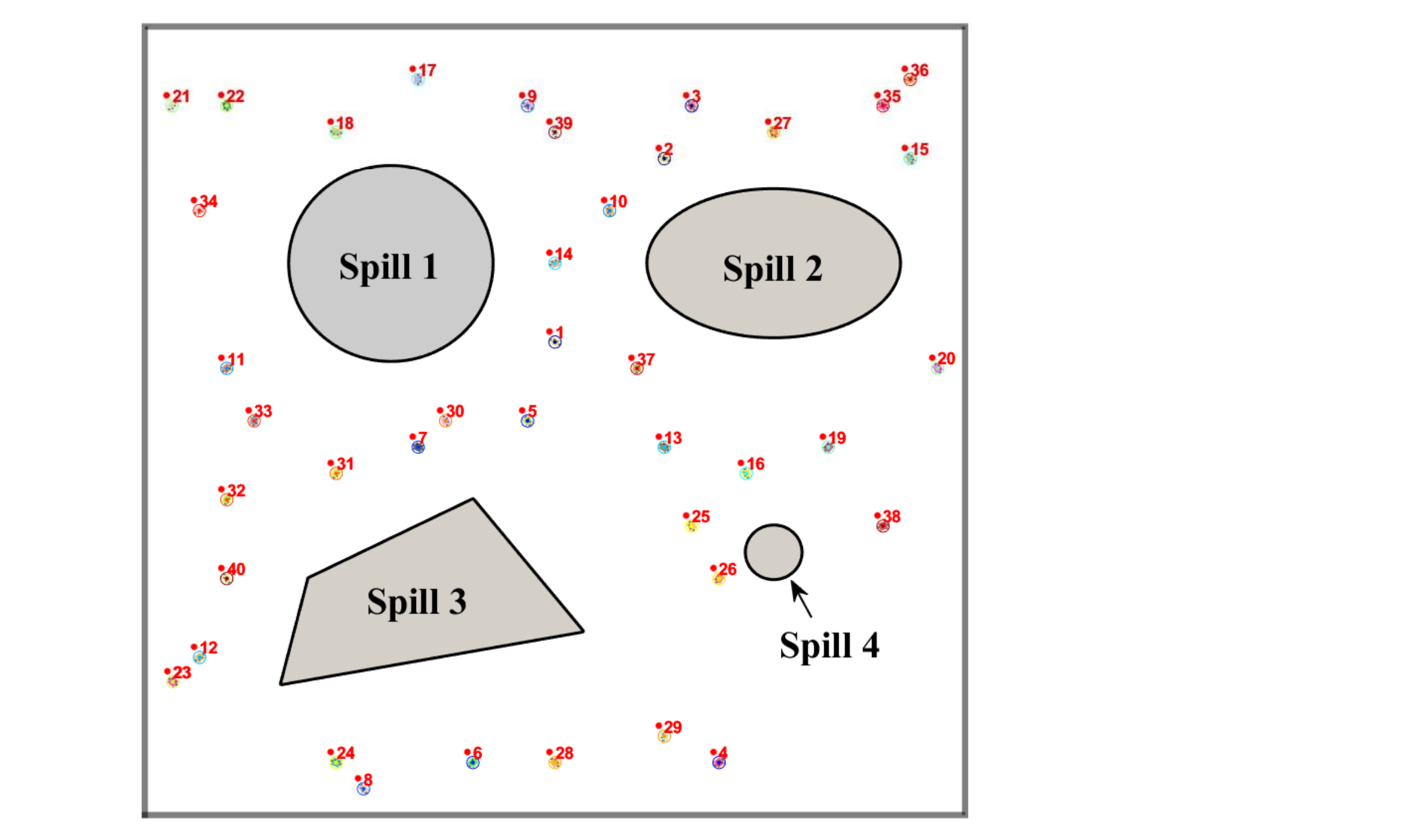}
\caption{The workspace for experiments with 4 spills featuring different shapes and areas. Randomly distributed robots are also indexed.}
\label{fig:simu_envrion_setting}
\end{figure}

\begin{table}[t]
\begin{center}
\caption{Geometry of the spills in the workspace $\mathcal{W}$.}
\label{tab:spill_areas}
  %\resizebox{\columnwidth}{!}{
  \begin{tabular}{|c |c |c| c| c|}
    \hline
    \bfseries Geometry & $Spill 1$ & $Spill 2$ & $Spill 3$ & $Spill 4$  \\ 
    \hline
    $Area (m^2)$ &  0.4301 & 0.4046 & 0.4200 & 0.0314 \\
    $boundary (m)$ & 2.3247 & 2.3926 & 2.8415 & 0.6283 \\
   \hline
   \end{tabular}
\end{center} 
\vspace{-0.2cm}
\end{table}

\subsection{Evaluation Metrics}
\label{sec:metrics}
The following three metrics were used to evaluate the performance of the proposed coverage control strategy in our experiments. Note that the units in simulations are in accord with Robotarium testbed, which is a scaled-down environment.
\begin{enumerate}
\item Lyapunov candidate function (Convergence):
\begin{equation}
L_x = \sum\limits_{R_i,R_j \in \mathcal{S}_x, i\neq j}  \|\textbf{q}_i - \textbf{q}_j\|.
\label{equ:convergence}
\end{equation}
If there is only one robot allocated for a spill $\mathcal{S}_x$, the Lyapunov candidate function is revised as:
\begin{equation}
L_x = \sum\limits_{R_i \in \mathcal{S}_x}  \|\textbf{q}_i - \textbf{c}_x\|, 
\label{equ:convergence_2}
\end{equation}
where $\textbf{c}_x$ is the computational centroid of the patch.
\item Total distance traveled by all robots after covering each spill (Distance):
\label{equ:distance}
\begin{equation}
D_{sum} = \sum\limits_{\mathcal{S}_x \in \mathcal{W}} \sum\limits_{k=1}^{k_{max}}  \|\textbf{q}_i(k+1) - \textbf{q}_i(k)\| .
\end{equation}
\item Number of iterations ($k_{stop} \leq k_{max}$) to reach the following stop condition:\\
Stop at $k_{stop}$ if the current area which is bounded by instant boundary satisfies
\begin{equation}
A_x \leq A_{min}, \, \forall \mathcal{S}_x \in \mathcal{W}. 
\label{equ:stop}    
\end{equation}
Here $A_{min}$ is defined to be 1\% of the original area of each spill. 
\end{enumerate}

\subsection{Experiment Cases and Results}
\label{sec:scenarios}
We validate our proposed solution in terms of efficacy, efficiency, scalability, and convergence. First, we select a case where a relatively large vision sensing range $r_{\mathcal{A}} = 1\text{m}$ is adopted such that every robot can detect the boundary within its vision range. Therefore, every robot itself becomes a rendezvous point $D_i, i \leq N$ and is designated to a specific spill by following a nearest-neighbor method. 
This case is to show that a larger vision range helps reduce rendezvous hierarchy and hence the total operation time.
The robots association to a spill is shown in Table~\ref{tab:scenario}. 
Meanwhile, as a contrast, in the second case, a smaller vision sensing range and a certain wireless communication range are chosen. In the second case, we use the same number of robots as the first case. We verify the performance of the proposed boundary shrink control method where a hierarchical rendezvous takes effect before boundary shrink control starts.

We validate our proposed strategy with many other experiments in terms of different operation ranges, different numbers of robots, etc. To be concise, the detailed results such as robot initial distributions and trajectories for each spill are omitted. Instead, the effects resulted by these factors are discussed in a comprehensive way in Sec.~\ref{sec:further_disc}.
All the experiments can be seen in complementary videos.

\subsubsection{Case 1 $(N = 40, r_{\mathcal{A}} = 1\text{m})$}
\label{sec:case_1}
In this case, the boundary shrink control with 40 robots deployed was tested. Fig.~\ref{fig:case_1_a} shows the random initial distribution of the 40 robots. Since the large vision sensing range $r_{\mathcal{A}}$ allows each robot being the rendezvous point, they have no child robot associated and can directly switch to boundary searching and tracking states. No hierarchical rendezvous or inter-robot communication is needed. Assume each spill in the workspace is detected by at least one robot, all the spills can be covered eventually. After removing the spill, the robots ultimately should converge at a point near the centroid of the spill as the boundary shrinks to zero. 
As documented in the top field of Table.~\ref{tab:scenario}, every spill was associated with many robots, except for Spill 4. The area evolution chart in Fig.~\ref{fig:case_1_c} shows that the completion progress of coverage was steadily declining. Since Spill 4 was detected with only one robot, its Lyapunov candidate function was defined according to (\ref{equ:convergence_2}) while other spills have Lyapunov candidate function defined according to (\ref{equ:convergence}). The convergence figure in Fig.~\ref{fig:case_1_d} also shows the sum of the Lyapunov candidate values of all the spills, i.e., $ \sum_{\mathcal{S}_x \in \mathcal{W}} L_x$, and therefore describes the overall performance of robot convergence in the workspace. 
According to the convergence figure, all the robots converged to their endpoints while covering, which was confirmed by the Lyapunov candidate function $L_x$ approaching to {\it zero}. In addition, because of the physical dimensional restraint of robots, the sum of convergence in this case did not amount to {\it zero} when the task approached its end. 
As depicted in Fig.~\ref{fig:case_1_peri}, the boundary, evaluated as perimeter, was shrinking asymptotically while spill coverage is in progress.
Fig.~\ref{fig:case_1_b} shows the trajectories of the 40 robots during boundary shrink control. From Fig.~\ref{fig:case_1_b}, we can see that twisted trajectories developed in $R_{13}$ and $R_{32}$ during the searching period. This happened because $R_{13}$ and $R_{32}$ were trying to avoid colliding with approaching robots $R_{16}$ and $R_{11}$, which were working in the tracking state and hence prioritized over the searching state in collision. This validated the effectiveness of our proposed boundary shrink control strategy in terms of collision avoidance.

% trim={<left> <lower> <right> <upper>}
\begin{figure}[!t]
\vspace{-0.5cm}
\centering
\begin{tabular}{cc}
\hspace{-0.5cm}
\subfigure[]{\label{fig:case_1_a}
\includegraphics[width=0.25\textwidth,trim={0 0 1cm 0},clip]{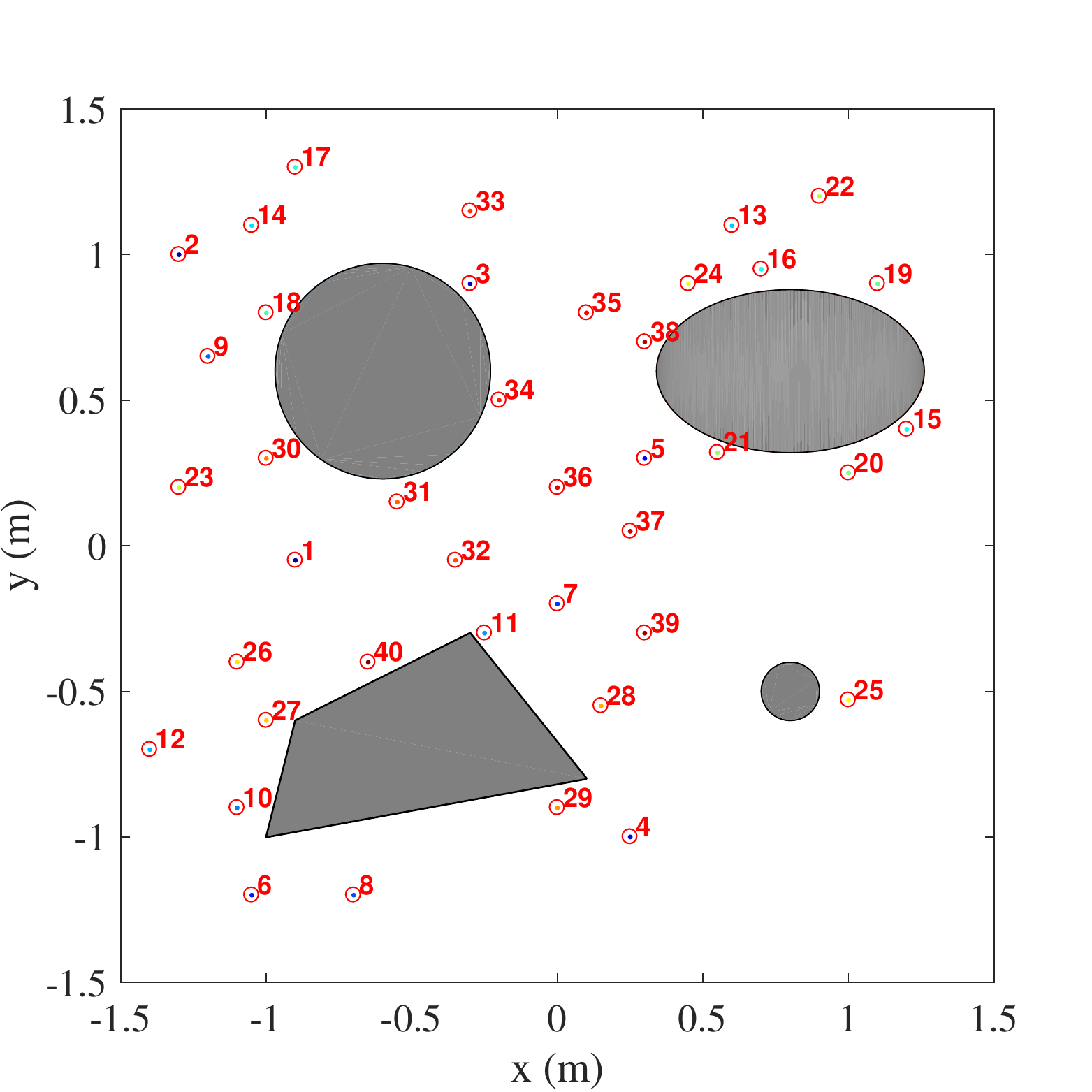}} &
\hspace{-0.5cm}
\subfigure[]{\label{fig:case_1_b}
\includegraphics[width=0.23\textwidth,trim={2.8cm -1.8cm 2.5cm 0.5cm},clip]{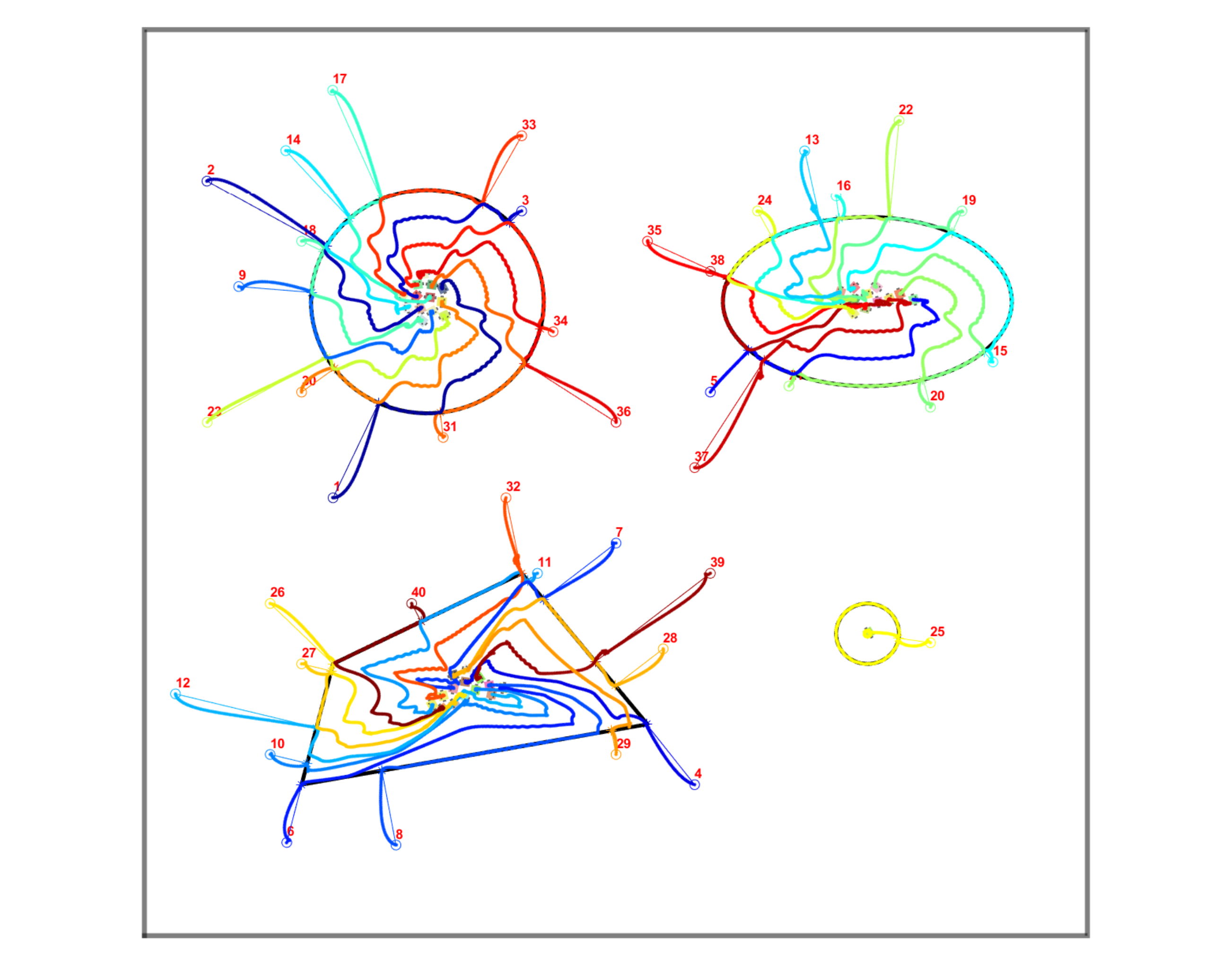}} 
\vspace{-0.3cm}\\
\hspace{-0.7cm}
\subfigure[]{\label{fig:case_1_c}
\includegraphics[width=0.24\textwidth,trim={0 0 13.3cm 0},clip]{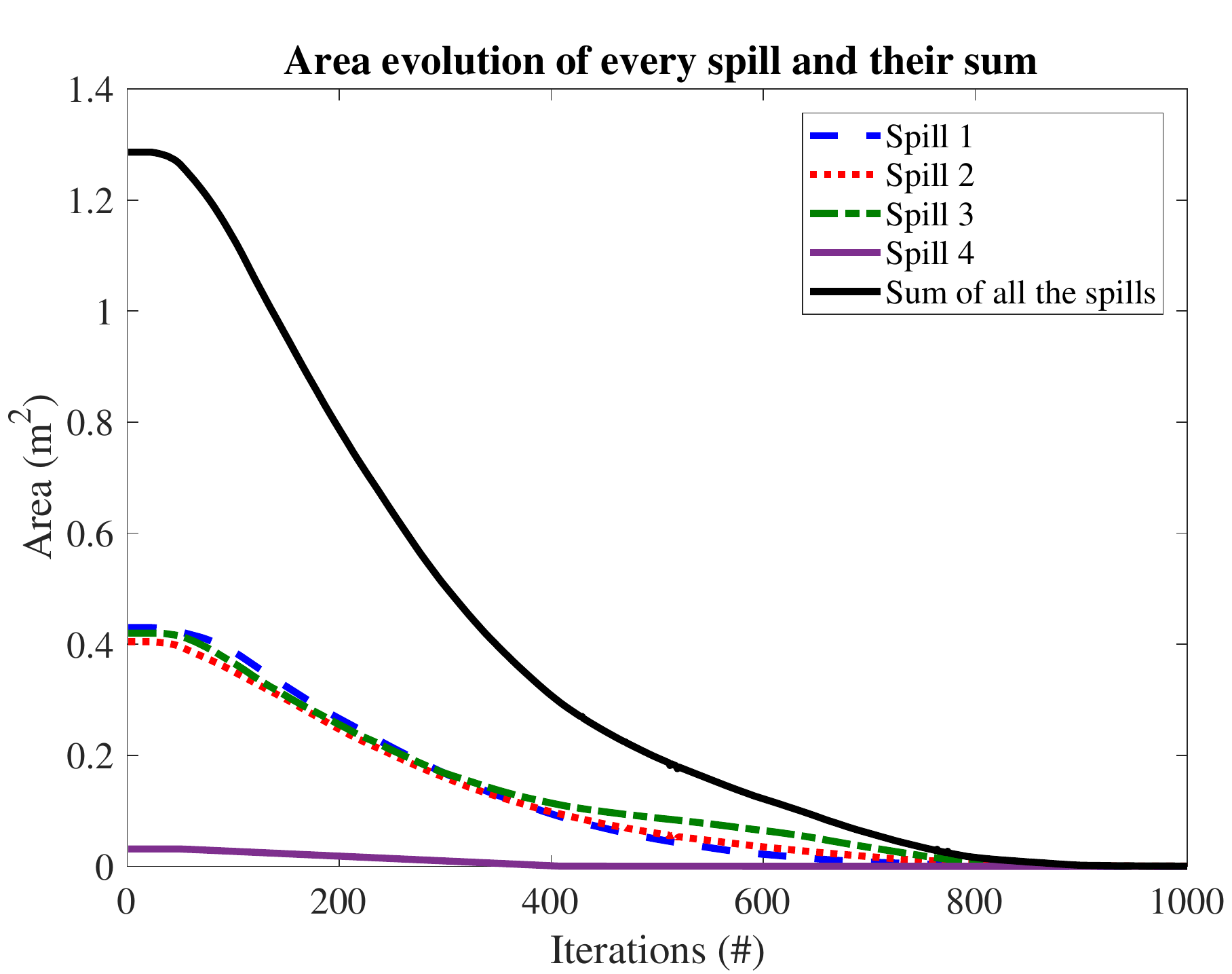}} &
\hspace{-0.9cm}
\subfigure[]{\label{fig:case_1_d}
\includegraphics[width=0.24\textwidth,trim={0 0 13.3cm 0},clip]{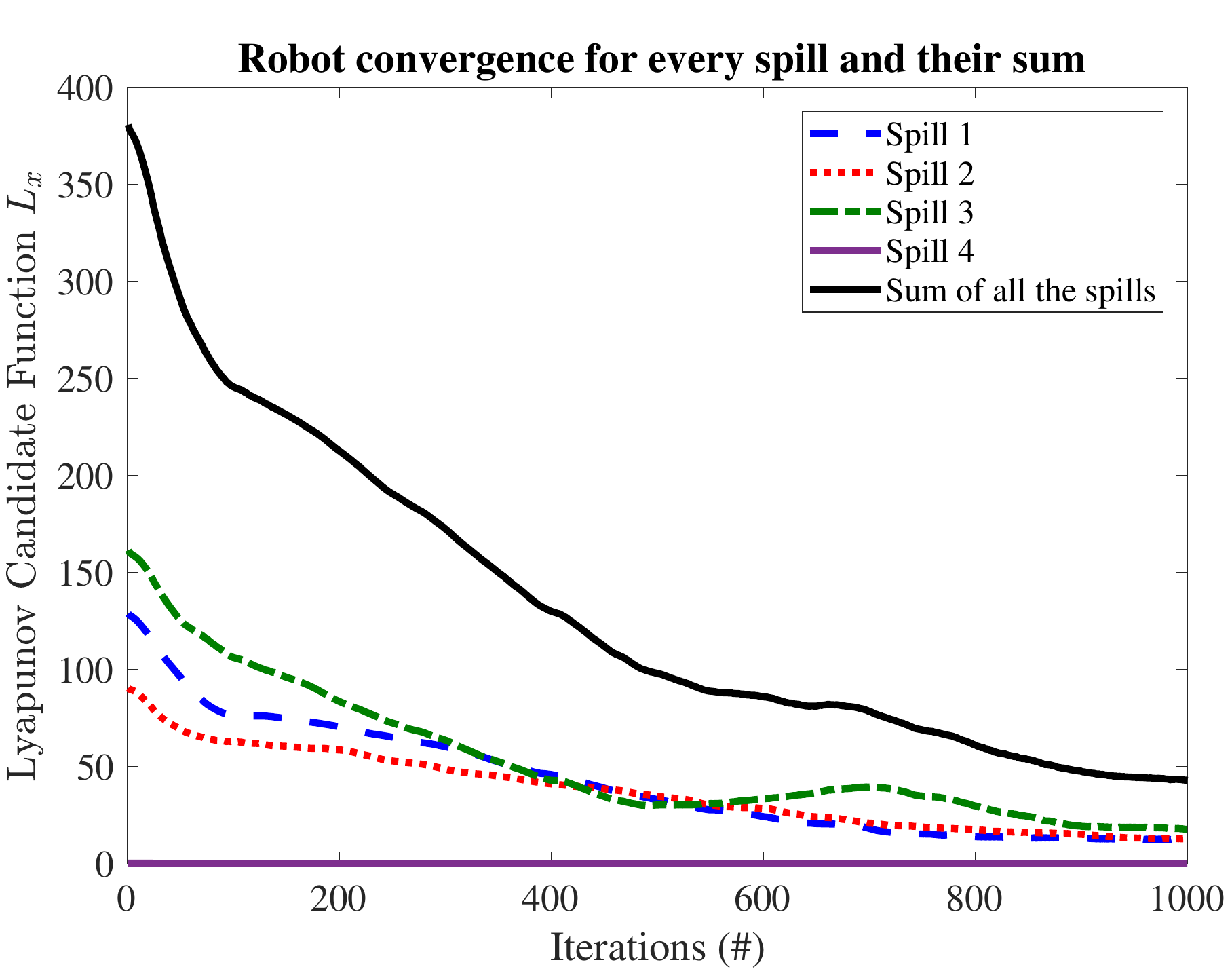}} 
\end{tabular}
\caption{(a) shows the initial random distribution of robots while (b) illustrates the coverage trajectories for Case 1 with 40 working robots. The area evolution and convergence with respect to time (iterations) are shown in (c) and (d), respectively.}
\label{fig:case_1}
\end{figure}

% trim={<left> <lower> <right> <upper>}
\begin{figure}[!t]
\vspace{-0.5cm}
\centering
\begin{tabular}{cc}
\hspace{-0.4cm}
\subfigure[]{\label{fig:case_1_peri}
\includegraphics[width=0.24\textwidth,trim={0 0 13.3cm 0},clip]{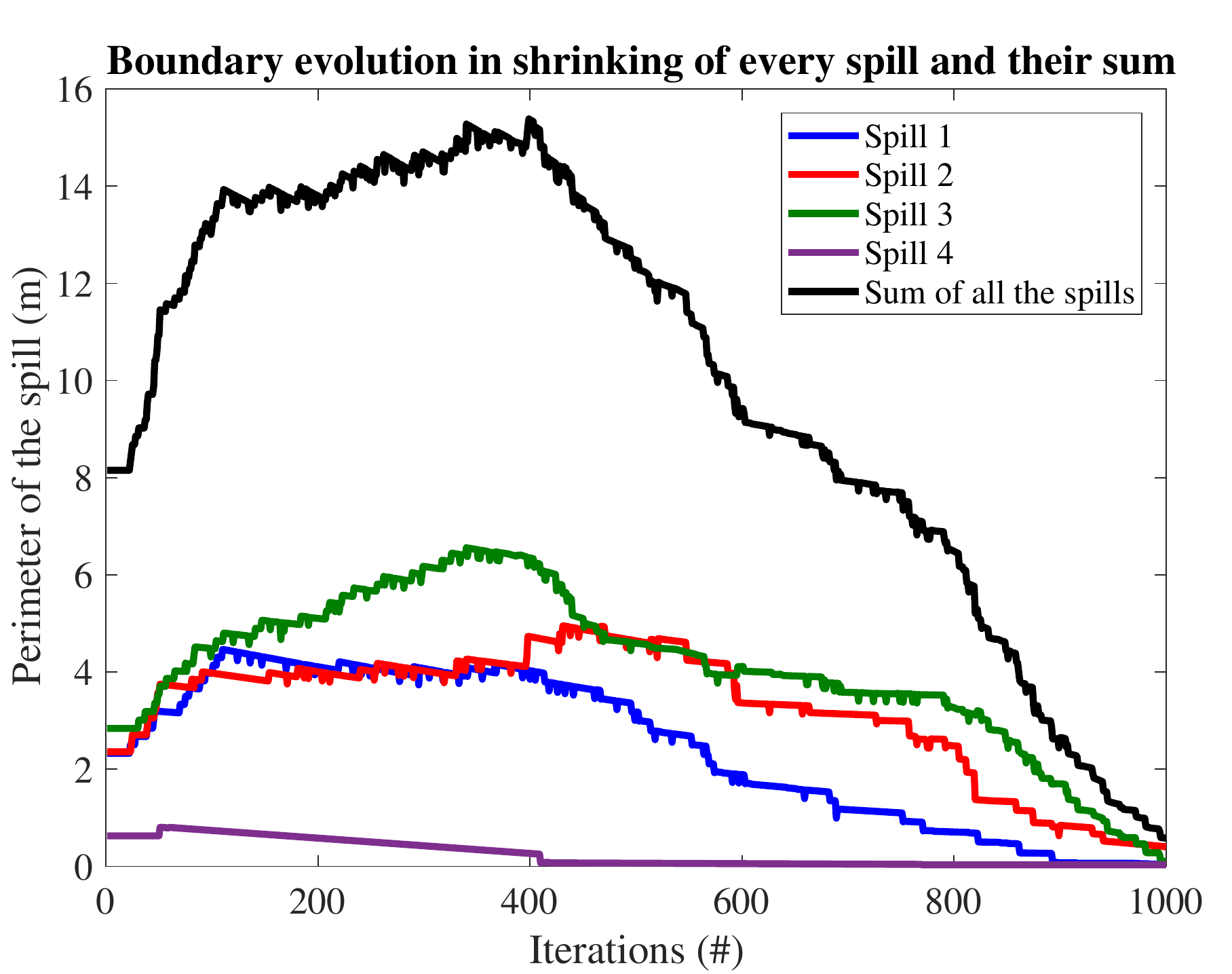}} &
\hspace{-0.6cm}
\subfigure[]{\label{fig:case_2_peri}
\includegraphics[width=0.24\textwidth,trim={0 0 13.3cm 0},clip]{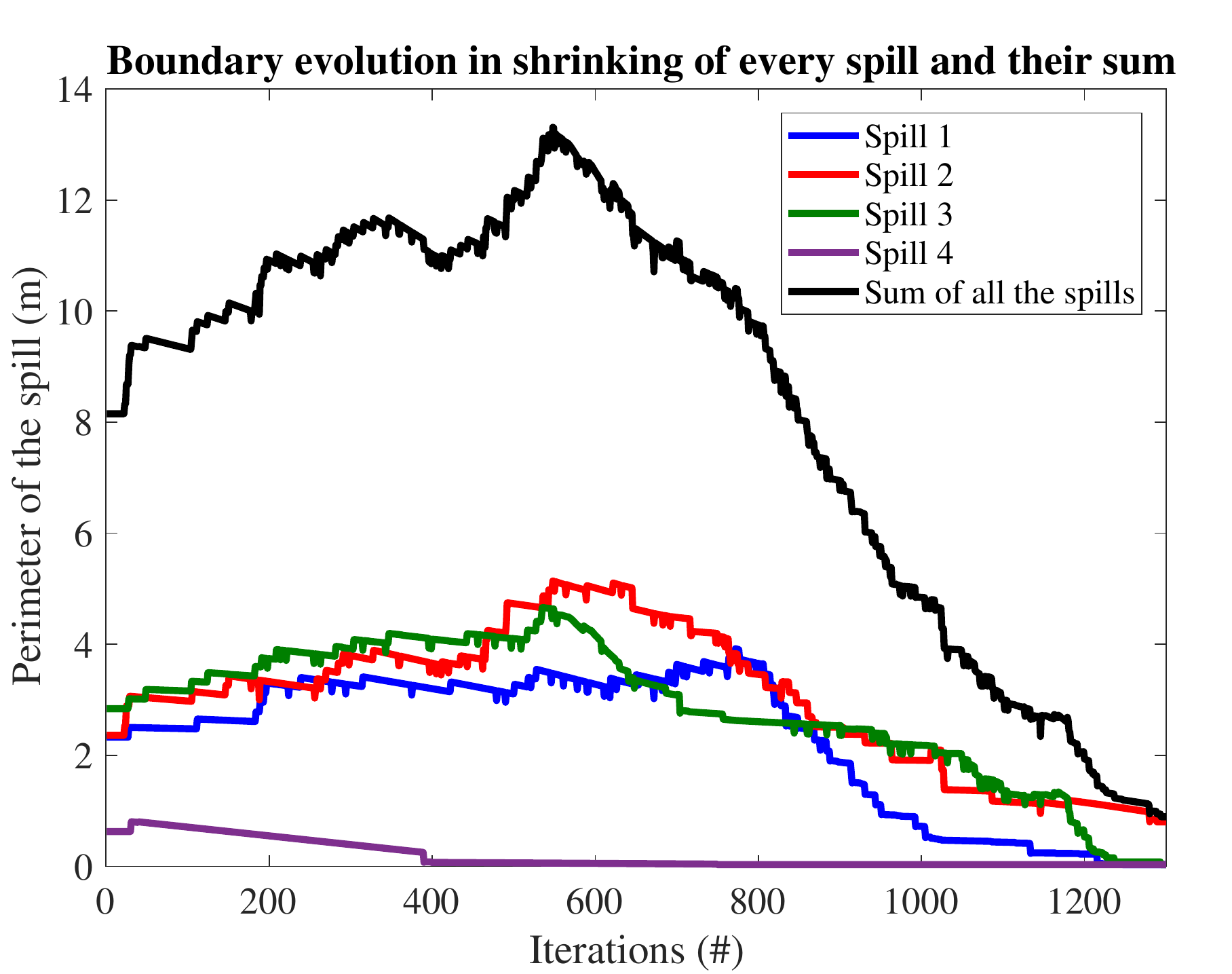}} 
\end{tabular}
\caption{The perimeter evolution is depicted separately of Case 1 (a) and Case 2 (b).}
\label{fig:peri_2_cases}
\end{figure}

\begin{table*}[!h]
\begin{center}
\caption{Performance of boundary shrink control in Case 1 and 2 having {\it N} = 40 robots, with corresponding maximum iterations $k_{max}$.}
\label{tab:scenario}
  \begin{threeparttable}[b]  
  %\resizebox{\columnwidth}{!}{
  \begin{tabular}{|c |c |c |c| c| c|}
    \hline
    \bfseries Case & \bfseries Metrics & \boldmath$Spill 1$ & \boldmath$Spill 2$ & \boldmath$Spill 3$ & \boldmath$Spill 4$  \\ [0.5ex] 
   \hline \hline
   \multirow{5}{12em}{\boldmath$Case \, 1$ \\ \boldmath$r_{\mathcal{A}} = 1\text{m}$ \\ \unboldmath$k_{max} = 1,000$} & $Residual\, area \, at\, k_{max} (m^2)$ &  $7.85\times 10^{-5}$ & $7.97\times 10^{-5}$ & $7.19\times 10^{-4}$ & $7.85\times 10^{-5}$ \\
    & $Completeness \, at\, k_{max} (\%)$ & 99.96 & 99.98 & 99.83 & 99.75 \\
    & $Number \, of \, allocated \, robots$ & 
     \begin{tabular}{@{}c@{}c@{}}13 \end{tabular}
       &  \begin{tabular}{@{}c@{}c@{}}12 \end{tabular}
       & \begin{tabular}{@{}c@{}c@{}}14 \end{tabular}  
       & 1 \\
    & $k_{stop} \, at \, A_{min} \, (\#)$ & 750 & 842 & 841 & 410 \\
  	& $D_{sum} (m)$& 2319.1
  &  2838.5 & 2712.8 & 23.6 \\
  \hline\hline
    \multirow{6}{12em}{\boldmath$Case \, 2$ \\ \boldmath$r_{\mathcal{A}} = 0.13\text{m}, r_{\mathcal{C}} = 0.3\text{m}$ \\ \unboldmath$k_{max} = 1,300$} & $Residual\, area \, at\, k_{max} (m^2)$ &  $2.85\times 10^{-4}$ & $7.94\times 10^{-5}$ & $7.85\times 10^{-5}$ & $7.85\times 10^{-5}$ \\
    & $Completeness \, at\, k_{max} (\%)$ & 99.93 & 99.98 & 99.98 & 99.75 \\
    & $Number \, of \, allocated \, robots$ & 13  & 12 & 14 & 1 \\
    & $Rendezvous \, point \, D_m \, (indices)$ & \begin{tabular}{@{}c@{}c@{}}5\,(\#3, \#18, \\ \#30, \#31, \#34) \end{tabular}
       &  \begin{tabular}{@{}c@{}c@{}}7\,(\#15, \#16, \#19, \\ \#20, \#21, \#24, \#38) \end{tabular}
       & \begin{tabular}{@{}c@{}c@{}}5\,(\#10, \#11, \\ \#27, \#29, \#40) \end{tabular}  
       & 1\,(\#25) \\ 
    & $k_{stop} \, at \, A_{min} \, (\#)$ & 924
  &  1093 & 1115 & 390 \\
  	& $D_{sum} (m)$ & 2420
  & 2283.2 & 2494.9 & 24.7 \\
  \hline 
   \end{tabular}
   \end{threeparttable}
\end{center}
\end{table*}

\subsubsection{Case 2 $(N = 40, r_{\mathcal{A}} = 0.13\text{m}, r_{\mathcal{C}} = 0.3\text{m})$}
\label{sec:case_2}
With a small vision range but relatively wide wireless communication range, the robots that are far from the spill can maneuver to its vicinity and catch sight of the boundary with the help of hierarchical rendezvous based on wireless connectivity. The initial robot distribution is set to be identical to Case 1, which can be found in Fig.~\ref{fig:case_1_a}. In this case, 18 shortest path trees are constructed according to Sec.~\ref{sec:rendezvous} with rendezvous points being $D_1(R_{3}), D_2(R_{18}), ...,\, \text{and} \, D_{18}(R_{25})$, representing the robots that already detect the boundary at first. The rendezvous path trees are shown in Fig.~\ref{fig:tree}. The numerical results including robot association to these spills, rendezvous points $D_m$, and experiment data are shown at the second row of Table~\ref{tab:scenario}. 

\begin{figure}
    \centering
    \includegraphics[width=0.95\columnwidth,trim={0 0 0 0},clip]{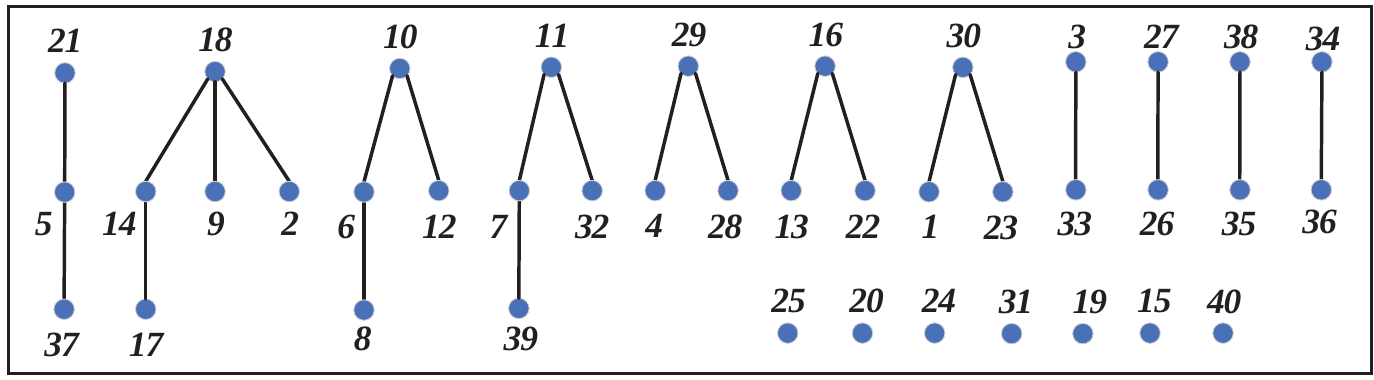}
    \caption{The shortest path trees that depict the hierarchical rendezvous, where the root node for each tree is a rendezvous point.}
    \label{fig:tree}
\end{figure}

The convergence curve was smoother but less sharp at the beginning in this scenario than in Case 1, as shown in Fig.~\ref{fig:case_2_d}, since multi-point rendezvous was taking effect but did not yield a faster convergence in general. Meanwhile, since part of the operation time was spent on rendezvous, which ended up until iteration $t=298$ in the experiment, the total coverage operation took longer time than that of Case 1. The boundary shrink control started after each robot moved and saw the boundary within its vision range, which happened right after robot completed rendezvousing to its corresponding rendezvous point. The spill areas were decreasing monotonically as shown in Fig.~\ref{fig:case_2_c} during the boundary shrink control, with operation range being $d=0.090\text{m}$. The robot trajectories are shown in Fig.~\ref{fig:case_2_b}, the child robot such as $R_{28}$ switched to boundary searching state at the time it met the spill boundary while approaching its rendezvous points $R_{29}$. 

The rendezvous hierarchy at the beginning of rendezvous is illustrated in Fig.~\ref{fig:tree}, which can be updated with local sensing. In Fig.~\ref{fig:tree}, robots such as $R_{37}$, $R_{17}$, etc., are leaf robots, meaning that they have no child robot associated. 
Noticeably, robots can switch to boundary searching state whenever they see the boundary within vision range, even if it happens before completion of rendezvous. This fact was verified by agents such as $R_{23}$, which first saw the spill boundary when it was moving toward $R_{30}$ and switched to boundary searching state as it had no child robot.
The parent robots except for rendezvous points (e.g. $R_{6}$ and $R_{14}$) can move while keeping the connectivity with their child robots (i.e., $R_{17}$ and $R_{8}$). 
Robots that become rendezvous points start moving to the boundary and perform boundary shrink control unless they have no child robot, such as $R_{25}$ and $R_{40}$, or all their child robots have switched away from rendezvous state.

\begin{figure}[!t]
\vspace{-1.1cm}
\centering
\begin{tabular}{cc}
\hspace{-0.5cm}
\subfigure[]{\label{fig:case_2_a}
\includegraphics[width=0.25\textwidth,trim={0 0 1cm 0},clip]{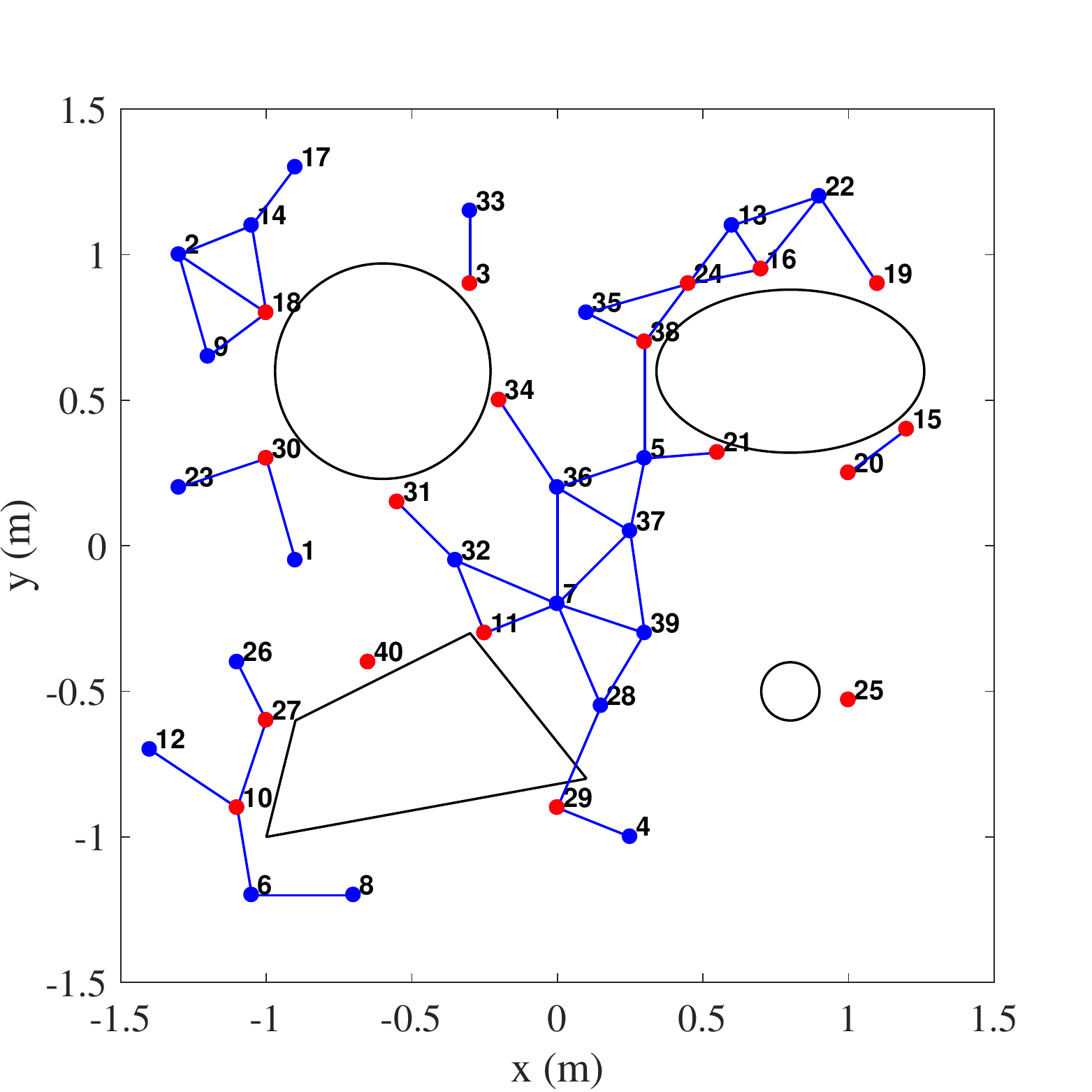}} &
\hspace{-0.5cm}
\subfigure[]{\label{fig:case_2_b}
\includegraphics[width=0.23\textwidth,trim={2.8cm -1.8cm 2.5cm 0.5cm},clip]{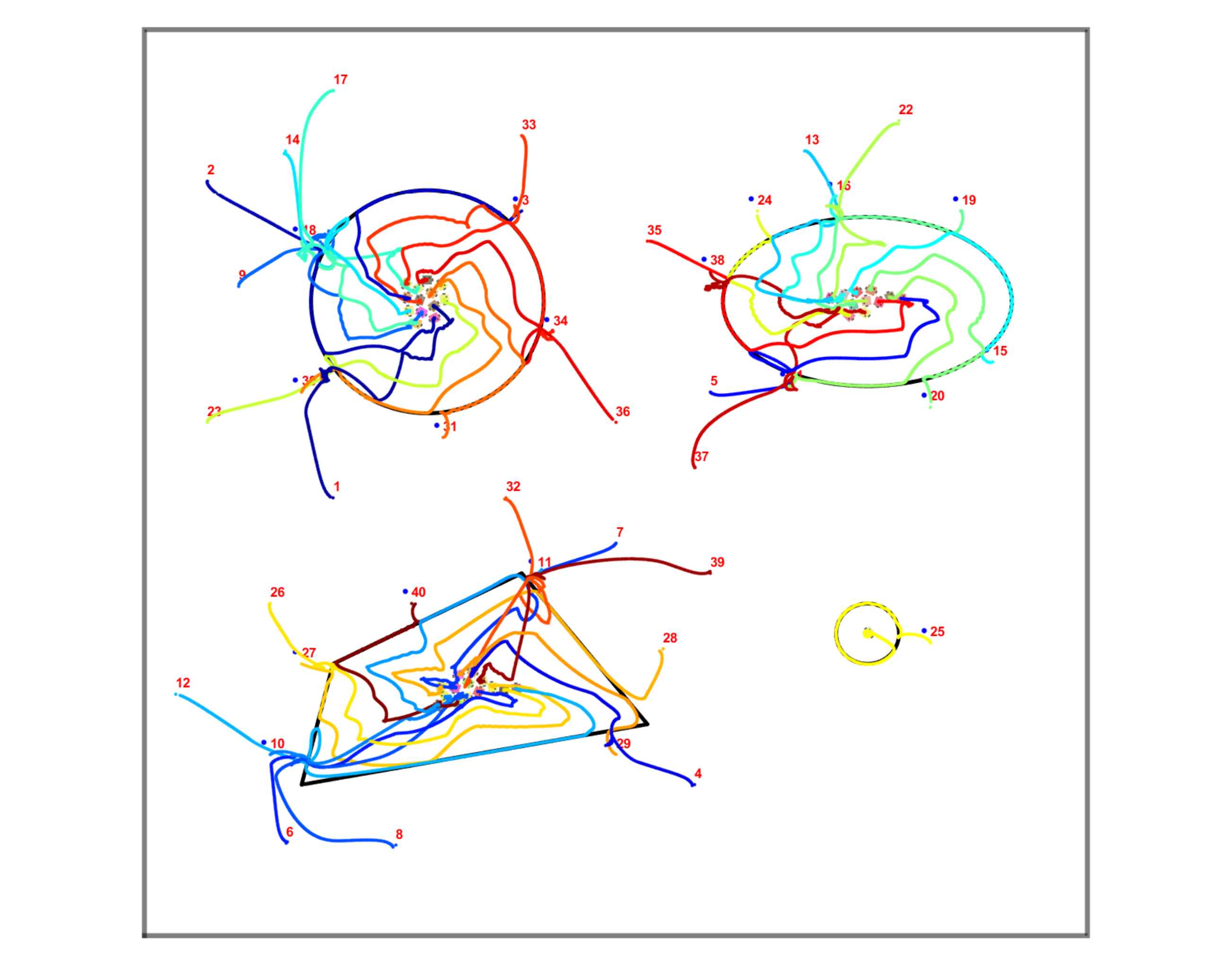}} 
\vspace{-0.3cm}\\
\hspace{-0.7cm}
\subfigure[]{\label{fig:case_2_c}
\includegraphics[width=0.24\textwidth,trim={0 0 13.3cm 0},clip]{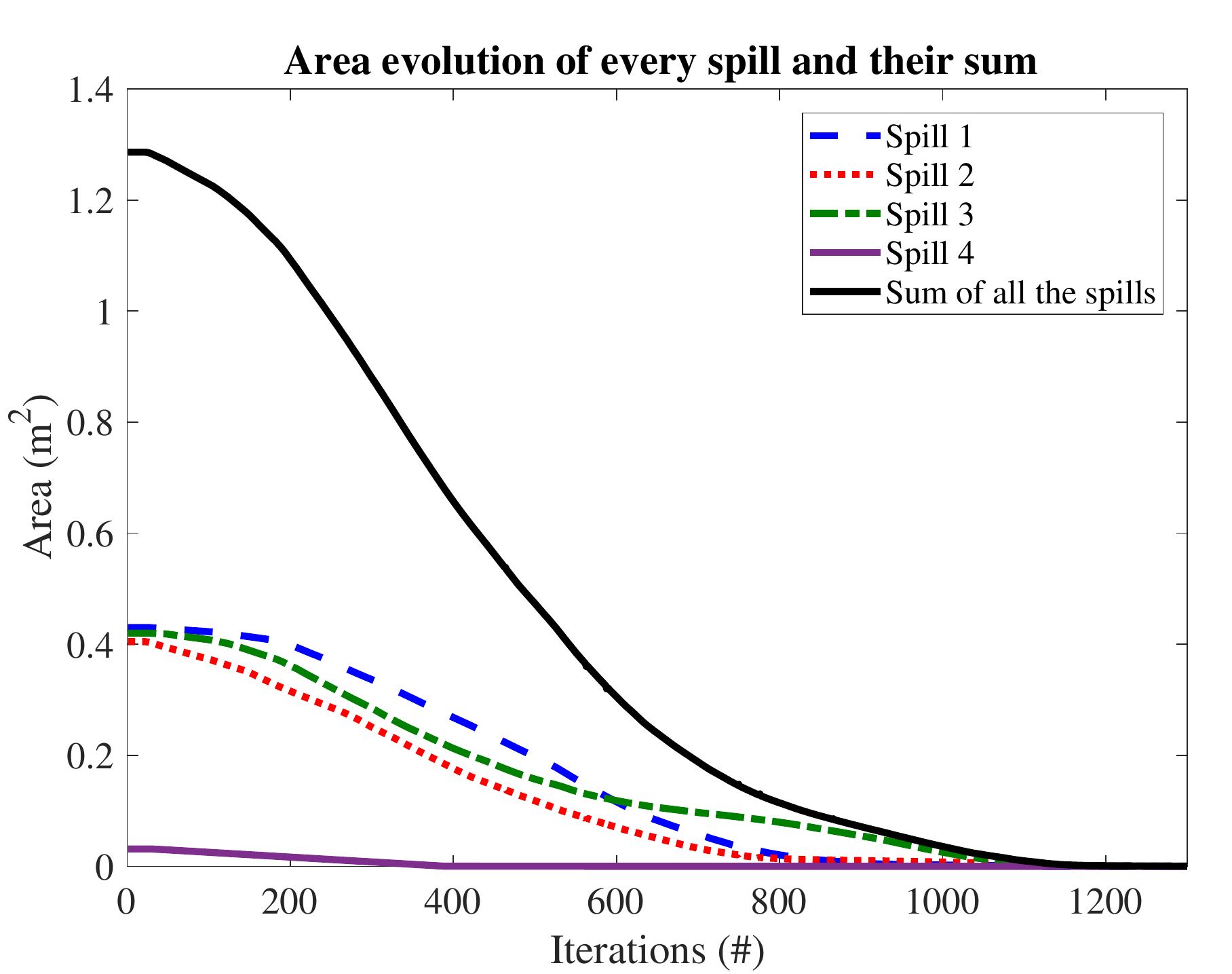}} &
\hspace{-0.9cm}
\subfigure[]{\label{fig:case_2_d}
\includegraphics[width=0.24\textwidth,trim={0 0 13.3cm 0},clip]{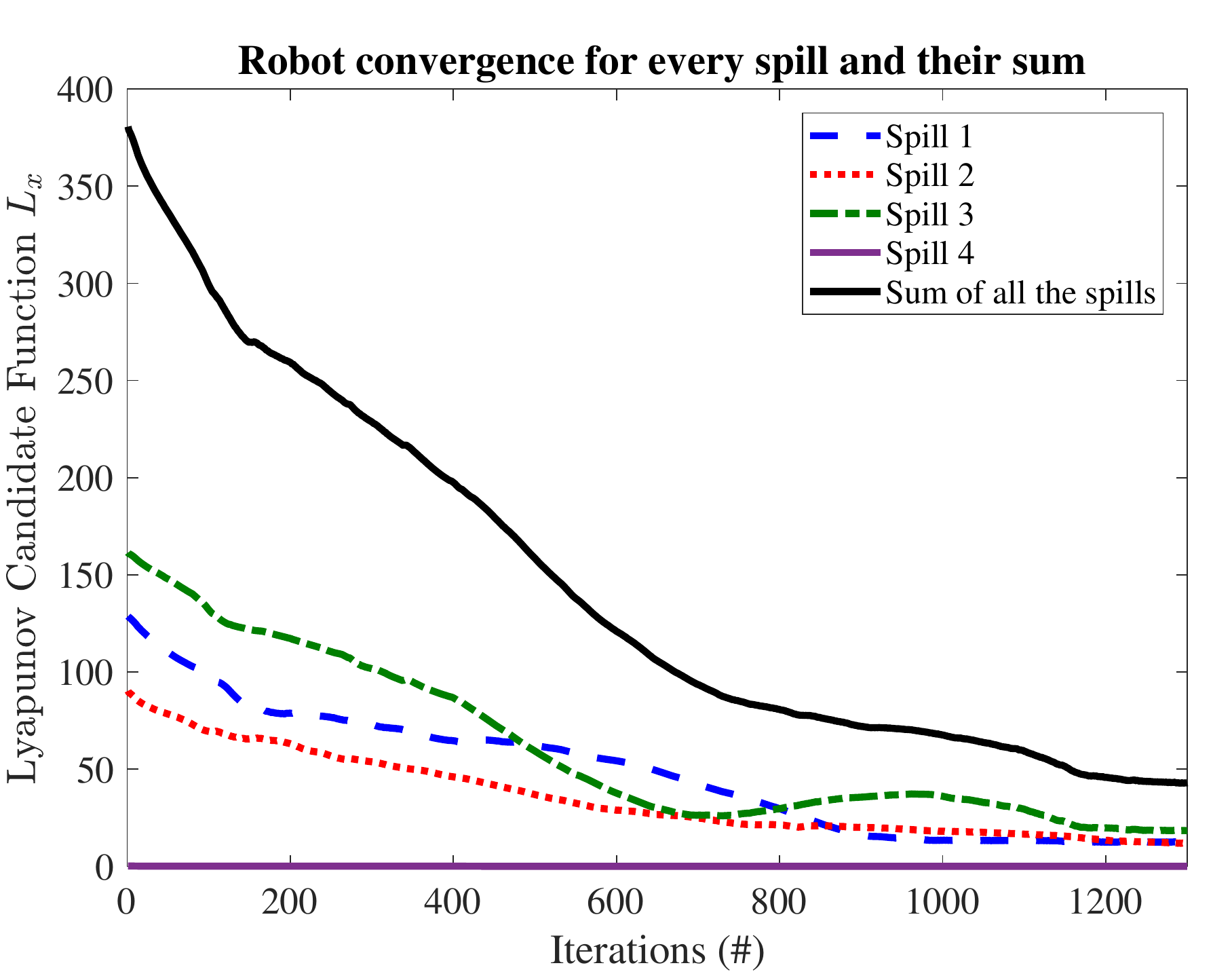}} 
\end{tabular}
\caption{(a) The same initial distribution of robots as Case 1, with wireless connectivity represented as a graph in solid blue lines and rendezvous points $D_m$ highlighted in red. (b) shows boundary shrink control trajectories with multi-point rendezvous. The area evolution and convergence of Case 2 are shown in (c) and (d), respectively.}
\label{fig:case_2}
\end{figure}

\subsection{Further Discussion}
\label{sec:further_disc}
Since the completion time of boundary shrink control can be reduced if having more working robots, we test a total of four scenarios in which robot numbers are increasing with the same interval, i.e., $N=\{30, 40, 50, \, \text{and}\, 60\}$, provided that $r_{\mathcal{C}}$ and $r_{\mathcal{A}}$ are the same as Case 2. 
Our analysis of these four scenarios reveals some significant outcomes. First, the total operation time decreased significantly by having more robots involved. More robots may lead to more rendezvous points and a fewer number of hierarchy levels, which accelerates rendezvous process. This result further substantiates the significance of our proposed multi-point rendezvous strategy. The area evolutions of the four spills as a whole for each scenario are depicted in Fig.~\ref{fig:area_for_N}.
Second, the operation time decreased much slower as we employed more than 50 robots, mostly because inter-robot collision avoidance has become an dominating issue that resulted in a drop in efficiency. 

Furthermore, we validate our solution for different operation ranges: $d=\{0.045\text{m}, 0.09\text{m}, 0.135\text{m},\, \text{and}\, 0.18\text{m}\}$, with the same initial robot distribution as Case 2 and appropriate range values. 
The area evolution of the entire workspace for these operation ranges is shown in Fig.~\ref{fig:area_for_d}. The videos of these experiments can be found in the paper video available at {\it \url{https://youtu.be/XUrKICAk4SQ}}.

Through these experiments, we validated the efficacy, efficiency, and scalability of our proposed solution. 

\begin{figure}[!t]
\vspace{-0.8cm}
\centering
\begin{tabular}{cc}
\hspace{-0.4cm}
\subfigure[]{\label{fig:area_for_N}
\includegraphics[width=0.24\textwidth,trim={0 0 13.3cm 0},clip]{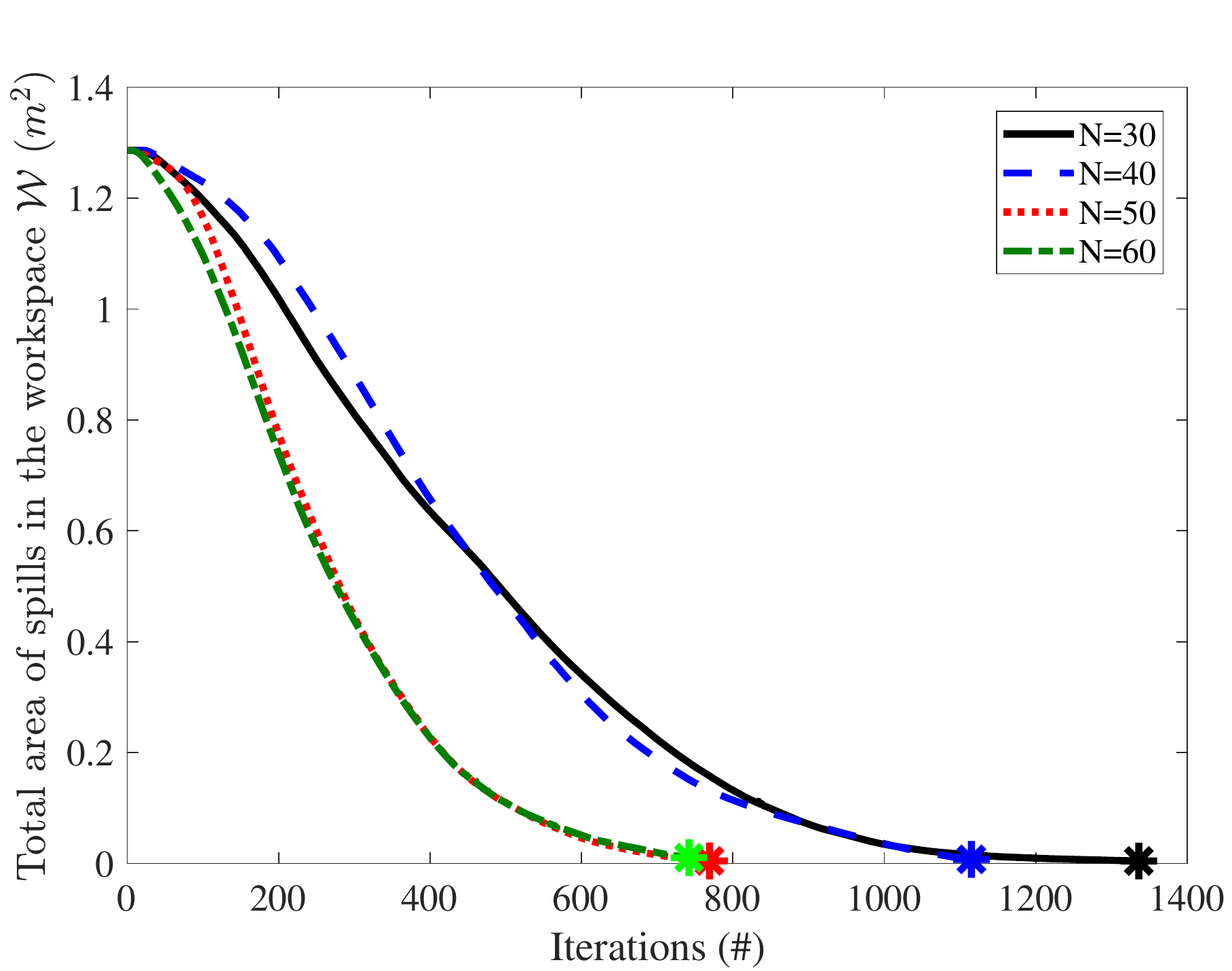}} &
\subfigure[]{\label{fig:area_for_d}
\hspace{-0.6cm}
\includegraphics[width=0.24\textwidth,trim={0 0 13.3cm 0},clip]{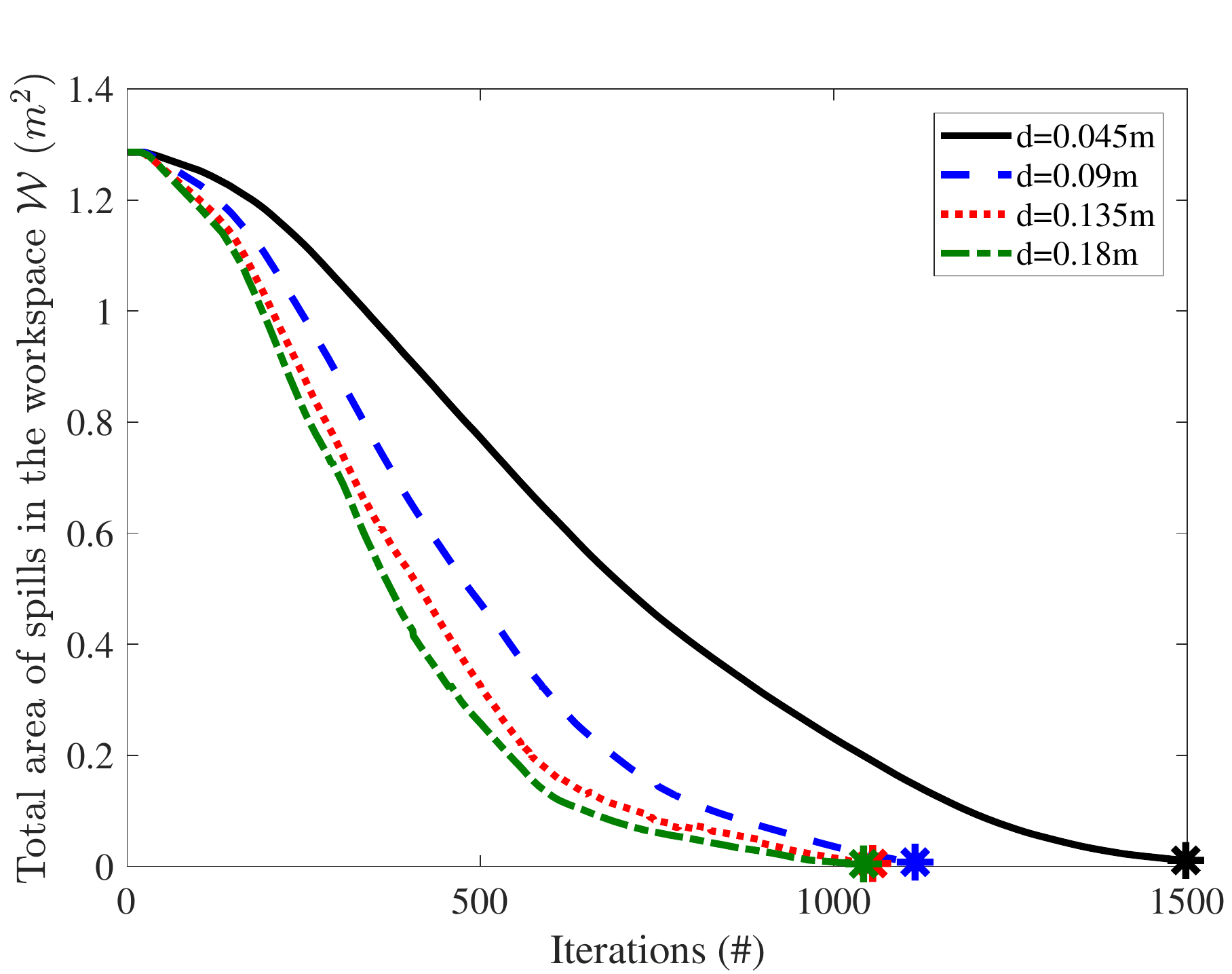}} 
\end{tabular}

\caption{(a) The change of residual areas in the workspace with respect to the iteration number for having different number of working robots $N$. The star marker at the end of every curve indicates its corresponding $\max\{k_{stop}\}$. (b) The evolution of residual areas for having different operation range $d$. The star markers indicated $\max\{k_{stop}\}$ for each scenario.}
\label{fig:Area_change}
\end{figure}

\section{Non-convex Spill}
\label{sec:nonconvex_spill}
The covered spill mentioned above which boundary shrink control performs on is assumed to be convex during operation. However, as can be seen from the experiment videos, Spill 3 failed to remain convex while robots were collecting the spill. Fortunately, the task was still completed because the spill was not completely broken. It is possible to overlook this situation and apply our solution to non-convex situations. However, the challenge is how to make sure every separate piece of spill has at least one robot associated with it when an original spill is broken into several pieces, since the narrowest place is eroded with coverage. Robot rendezvous can be executed to reallocate the robots or create allocation methods that promote robot dispersion and prevent many robots from gathering at one piece \cite{batalin2002spreading}. Enabling ad hoc networking among robots may also help to achieve convergence while the robots are assembling for coverage. Research on this topic is to be further explained in our future works.

\section{Conclusions and Future Work}
\label{sec:conclusions}
This paper proposes a concrete boundary shrink control strategy that can be used in scenarios requiring collective coverage such as oil spill cleanup, complete exploration of an environment, and cooperative de-mining. The system does not rely on any method of global localization; its distributed mechanism promises scalability and robustness to varied scenarios. The boundary shrink control is implemented as a hybrid solution consisting of three states: {\it rendezvous}, {\it searching}, and {\it tracking}. The convergence and stability of this hybrid control are proved mathematically.
Extensive experiments in simulation have verified the effectiveness of our solution in concurrently covering multiple spills under many practical constraints, and the adoption of different parameter configurations demonstrated the system's scalability and flexibility.

Future work will investigate how the proposed boundary shrink control method can be adapted to non-convex spills. We will also examine how to decide the optimal numbers of robots for given spill shapes and distributions. In addition, more features such as fault tolerance will be introduced with the primary goal of improving the system's robustness and efficiency in complex environments.

%%%%%%%%%%%%%%%%%%%%%%%%%%%%%%%%%%%%%%%%%%%%%%%%%%%%%%%%%%%%%%%%%%%%%%%%%%%%%%%%%%%%%

\bibliographystyle{IEEEtran}
\bibliography{main}

\begin{IEEEbiography}[{\includegraphics[width=1in,height=1.25in,clip,keepaspectratio]{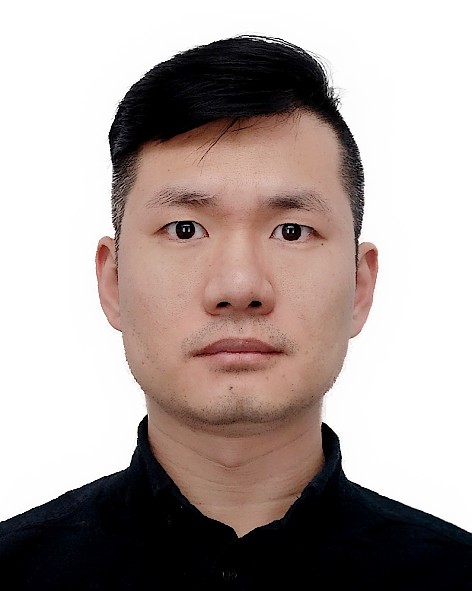}}]%
{Shaocheng Luo} received his B.S. degree in Mechanical Engineering from Harbin Institute of Technology in China in 2009, and the M.S. degree in Mechatronic Engineering from Zhejiang University in China in 2012. He received his Ph.D. degree in technology with a specialization in Robotics from Purdue University in West Lafayette, Indiana, USA, in 2020. His research interests span the areas of multi-robot systems, wireless networks, robotic system applications in environmental monitoring and operations, and cyber-physical systems.
\end{IEEEbiography}
\vskip 0pt plus -1fil

\begin{IEEEbiography}[{\includegraphics[width=1in,height=1.25in,clip,keepaspectratio]{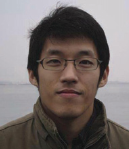}}]%
{Jonghoek Kim} received the B.S. degree in electrical and computer engineering from Yonsei University, South Korea, in 2006, the M.S. degree in electrical and computer engineering from the Georgia Institute of Technology, in 2008, and the Ph.D. degree, co-advised by Dr. F. Zhang and Dr. M. Egerstedt, from the Georgia Institute of Technology, Atlanta, GA, USA, in 2011. He has worked as a Senior Researcher with the Agency for Defense Development, South Korea, from 2011 to 2018. He is currently an Assistant Professor with the Department of Electrical and Computer Engineering, Hongik University, South Korea. His research is on target tracking, control theory, robotics, multiagent systems, and optimal estimation.
\end{IEEEbiography}
\vskip 0pt plus -1fil

\begin{IEEEbiography}[{\includegraphics[width=1in,height=1.25in,clip,keepaspectratio]{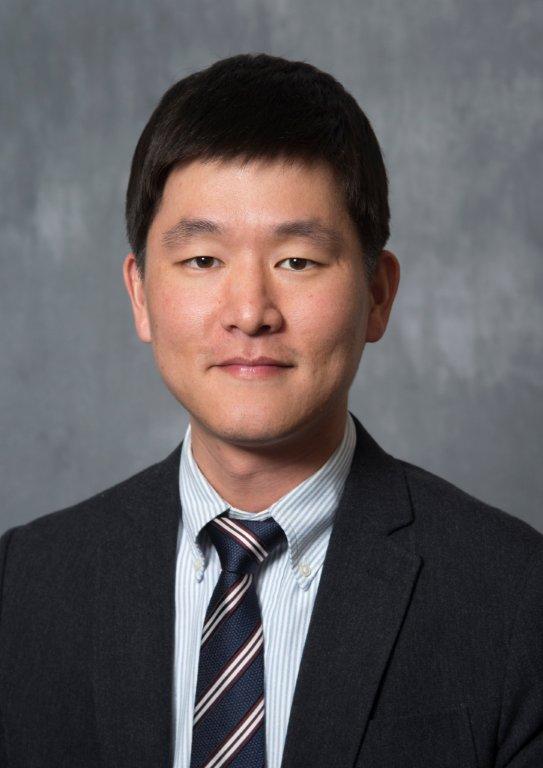}}]%
{Byung-Cheol Min} (M’14) received the B.S. degree in electronics engineering and the M.S. degree in electronics and radio engineering from Kyung Hee University, Yongin, South Korea, in 2008 and 2010, respectively, and the Ph.D. degree in technology with a specialization in robotics from Purdue University, West Lafayette, IN, USA, in 2014. 
 
He is an Assistant Professor of Department of Computer and Information Technology and the Director of the SMART Laboratory with Purdue University, West Lafayette, IN, USA. Prior to this position, he was a Postdoctoral Fellow with the Robotics Institute, Carnegie Mellon University, Pittsburgh, PA, USA. His research interests include multi-robot systems, human-robot interaction, robot design and control, with applications in field robotics and assistive technology and robotics.
 
He is a recipient of the NSF CAREER Award in 2019, Purdue PPI Outstanding Faculty in Discovery Award in 2019, and Purdue CIT Outstanding Graduate Mentor Award in 2019.

\end{IEEEbiography}

\end{document}